\documentclass[11pt]{article}

\usepackage[a4paper,margin=1in]{geometry}
\usepackage[T1]{fontenc}
\usepackage[utf8]{inputenc}
\usepackage{mathptmx}
\usepackage{amsmath,amssymb,amsthm}
\usepackage{graphicx}
\usepackage{booktabs}
\usepackage{array}
\usepackage[ruled,vlined,linesnumbered]{algorithm2e}
\usepackage[colorlinks=true,linkcolor=blue,citecolor=blue,urlcolor=blue]{hyperref}
\usepackage{cleveref}
\usepackage{microtype}
\usepackage{placeins}

\newtheorem{definition}{Definition}[section]
\newtheorem{theorem}{Theorem}[section]

\newcommand{\E}{\mathbb{E}}

\title{When and Why LLM Causal Priors Help:\\ Closed-Loop Prior Selection for Amortized Causal Inference}

\author{
  Haohao Zhou\\
  \texttt{National Key Laboratory of Information Systems Engineering}\\
  \texttt{National University of Defense Technology, Changsha, Hunan, China}\\
  \texttt{haohaozhou@nudt.edu.cn}
}

\date{}

\begin{document}

\maketitle

\begin{abstract}
Causal effect estimation asks how an outcome would change under an intervention, and medicine, economics, and public policy all treat it as a foundational task. Prior-data fitted networks (PFNs) amortize the task: a model trained on large numbers of programmatically generated synthetic causal tasks reads a new problem's observational data into context and returns an interventional-effect estimate in a single forward pass. The capability of such models is largely determined by the synthetic training prior, which is currently designed by hand, a bottleneck acknowledged by both Do-PFN and CausalPFN. Large language models (LLMs) can now ``draw'' plausible causal graphs for a given domain, suggesting that LLM-distilled graphs could serve as prior material. Whether injecting such graphs helps at all, where any gain comes from, and when injection helps. Practice has so far relied on manual trial and error. We propose a \emph{closed-loop prior selection framework} that casts prior injection as a budget-constrained optimization over a candidate prior pool. Candidates undergo cheap post-training and are scored by a composite metric dominated by real-domain generalization; the winner then receives full training and paired statistical validation. On a 7.34M-parameter Do-PFN, the framework's winner attains a formally significant $2.75\times$ gain on the primary evaluation domain ($n=5$ paired seeds, $p=0.0086$), and its error falls below that of the uninjected official base. Because the winner is built on a locally retrained base and the official base is not part of the paired test, this is a descriptive cross-lineage comparison. Generalization on an adjacent monitoring domain improves significantly ($p=0.0109$), and no monitored capability degrades. Mechanism experiments show that the gain depends on the semantic content of the distilled graph rather than its structural diversity alone does not produce it (directional evidence). The benefit has limits. Nine controlled experiments delineate its boundary and support a \emph{three-condition empirical regularity}: injection yields significant gains only when the base is underfit on the task domain, the prior domain matches the task domain, and the task lies within the support of the base's training prior. Two negative results carry general lessons: averaging multiple distillation sources dilutes gains, and synthetic probe metrics systematically diverge from real-domain transfer. With this framework and this regularity in hand, the use of LLM causal priors stops being manual trial and error and becomes an empirically verifiable selection problem.
\end{abstract}

\section{Introduction}

Decision-makers act on what an intervention would produce, and they rarely hold that information in hand, whether the question is whether a patient would recover better under a different treatment or how sales would respond to a different price. Causal effect estimation provides the statistical machinery for such questions \cite{pearl2009causality,rubin1974estimating}. Its core difficulty is that only naturally occurring data can be observed: interventional outcomes are expensive or ethically unavailable, so effects must be identified from observational data under additional assumptions such as ignorability. Classical methods model each new problem separately, at high cost and with little reuse.

Amortized causal inference takes a different path. A prior-data fitted network (PFN) \cite{muller2022transformers} learns inference over a task distribution (the prior) by training on large numbers of programmatically generated synthetic tasks, and at inference time it conditions on context and answers directly. Do-PFN \cite{robertson2025dopfns} instantiates the idea for causal effect estimation by encoding the intervention operator into prior sampling, so that tasks are generated as ``observational data $\to$ outcome under intervention''; CausalPFN \cite{balazadeh2025causalpfn} extends the task coverage. Once trained, the model reads a new problem's observational data into context and yields an interventional-effect estimate in one forward pass, with no per-problem optimization. Both models share a decisive property. Which causal systems they can handle is almost entirely determined by the synthetic training prior; the prior's coverage is the model's capability boundary. Because the prior is designed by hand, that coverage is limited by the designer's domain knowledge, and both Do-PFN and CausalPFN list prior design as an explicit limitation \cite{robertson2025dopfns,balazadeh2025causalpfn}.

Large language models (LLMs) offer a way past this bottleneck. Systematic evaluations show that LLMs can produce causal graphs that agree with domain knowledge to a considerable degree \cite{kiciman2023causal}, making them an off-the-shelf, on-demand source of causal knowledge. LLM-distilled causal graphs can then serve as prior material: training tasks generated from them carry domain semantics, and an existing base can continue to train on them, freeing amortized causal inference from hand-designed priors. Distilled graphs, though, are not reliable ground truth. They are highly sensitive to prompt wording, decoding temperature, and graph scale \cite{kiciman2023causal}, so their value and risk must be measured by injection experiments; graph-level fidelity alone does not establish them.

This possibility has never been systematically validated. Whether injection helps at all is unknown, and so is the source of any gain: does it come from the graph's semantic content, or merely from narrowing the training distribution? If narrowing dominates, LLM causal knowledge carries little real value. Equally open is whether gains can be separated from collateral costs, that is, whether some configuration improves the primary task without harming adjacent capabilities. And when several candidate graphs coexist, drawn from different distillation sources or decoding temperatures, there is as yet no systematic way to decide which one to inject, or to automate that decision. In practice, manual trial and error settles these questions. \emph{That gap defines the problem this paper addresses.}

We address these questions with a \emph{closed-loop prior selection framework}, a two-stage procedure. In the first round every candidate undergoes cheap post-training and is ranked by a composite score dominated by real-domain transfer. In the second round the top candidates receive full training and paired statistical validation, and this round delivers the final verdict. The base's architecture and training algorithm stay untouched throughout. What changes is the status of ``which prior to inject'': a manual decision becomes an optimizable, verifiable object.

On a 7.34M-parameter Do-PFN, the framework's winner attains a formally significant $2.75\times$ gain on the primary evaluation domain ($n=5$ paired seeds, $p=0.0086$), and its error level falls below that of the uninjected official base. This is a descriptive cross-lineage comparison: the winner is built on a locally retrained base, and the official base is not included in the paired test (see Sections~\ref{sec:m4} and \ref{sec:m6}). Generalization on the adjacent monitoring domain improves simultaneously (paired test $p=0.0109$), and no monitored capability degrades (descriptive). The gain, however, is not universal. Nine controlled experiments further show (i) that injecting the same recipe into an already-strong base degrades it $4$--$9\times$ with a complete ranking reversal, (ii) that domain-specific priors are significantly harmful on out-of-domain benchmarks, and that even domain-matched distillation cannot repair a benchmark whose dimensionality far exceeds the training-prior support. We therefore induce a \emph{three-condition empirical regularity}: injection yields significant gains only when the base is underfit, the prior domain matches the task domain, and the task lies within the training-prior support (full statement, evidence, and inductive limits in Section~\ref{sec:discussion}). Mechanism experiments (Section~\ref{sec:m5}) further show that the gain comes from semantic content drives the gain; structural diversity alone does not (directional evidence).

\paragraph{Contributions.}

Five lines of work feed directly into this paper: PFN-family amortized inference; LLMs as causal-knowledge sources; data-side intervention routes; the closed-loop search and self-improvement family; and the baseline spectrum and standard benchmarks.
\begin{itemize}
  \item \textbf{C1 (Framework).} We cast prior injection as a budget-constrained optimization (Equation~\eqref{eq:selection}) and provide a two-stage closed-loop selection procedure (cheap-proxy screening plus full paired validation), turning ``which prior to inject'' from a manual decision into an evaluably optimizable object. The procedure inherits the two-stage ``candidate pool + cheap proxy + full validation'' paradigm established by neural architecture search and hyperparameter optimization (Section~\ref{sec:related-search}). Our increment is to apply it to prior-distribution selection and to provide a controlled empirical delineation of its boundary.
  \item \textbf{C2 (Main result).} The framework's winner delivers a formally significant $2.75\times$ gain on the primary domain ($n=5$, $p=0.0086$), with an error level below the uninjected official base (descriptive cross-lineage comparison, Sections~\ref{sec:m4}, \ref{sec:m6}), while significantly improving adjacent-domain generalization ($p=0.0109$) without eroding existing capability (descriptive).
  \item \textbf{C3 (Mechanism and theory).} We give an evidence chain for ``content beats diversity'': the random-graph control, the V-shaped temperature response, and the reversal of the cheap-screening ranking (directional evidence). Three theorems give this picture conceptual grounding: an ensemble-error lower bound, probe--transfer ranking reversal, and unbounded risk outside the support. They motivate, respectively, selection over ensembling, real-domain-dominated proxies, and a support boundary on repair.
  \item \textbf{C4 (Applicability boundary).} Nine modules delineate the boundary and induce the three-condition empirical regularity (base underfit, prior-domain match, task within support), including two honest negative boundaries (cross-domain negative transfer, out-of-support high-dimensional non-repair), and make explicit the inductive limits (each condition rests on a single instance, conditions two and three are not separable, no quantitative prediction), turning ``when does injection help'' into an operational criterion.
  \item \textbf{C5 (External validity).} We run a same-protocol, per-domain comparison against a direct SOTA and a classical-estimator panel. The primary domain wins against all baselines including the SOTA ($4.1\times$/$6.1\times$ over CausalPFN); the secondary domain is an honest boundary (Section~\ref{sec:m8}), and two standard causal benchmarks give two honest boundaries (Section~\ref{sec:m9}).
\end{itemize}

\section{Related Work}

\subsection{PFN-family amortized inference}

Amortized inference rests on a simple exchange: one training phase in place of per-problem optimization. A model first learns inference over a parameterized task distribution, then at inference time conditions on context and answers directly \cite{muller2022transformers}. TabPFN \cite{hollmann2023tabpfn,hollmann2024tabular} brought this paradigm to tabular supervised learning, showing that in the small-sample regime a model pretrained on synthetic tasks can match or exceed classical methods tuned per dataset. A subsequent line extends the inference target from prediction to causation: Do-PFN \cite{robertson2025dopfns} encodes the intervention operator into prior sampling so that tasks are generated as ``observational data $\to$ outcome under intervention'', and the model outputs an interventional-effect estimate in one forward pass; CausalPFN \cite{balazadeh2025causalpfn} extends the causal task coverage (pretraining recipe details follow the original papers). This family shares one structural constraint: \emph{the prior determines the capability boundary}. In all existing work the prior has been designed by hand, and both Do-PFN and CausalPFN list prior design as an explicit limitation \cite{robertson2025dopfns,balazadeh2025causalpfn}. This paper operates at precisely this interface. We do not change the model architecture or training algorithm; we turn ``where the prior comes from and which one to pick'' into an optimizable object.

\subsection{LLMs as causal-knowledge sources}
\label{sec:related-llm}

Evaluation studies of LLM causal capability (often grouped as LLM4CD \cite{kiciman2023causal}) run along two lines: generation of causal graphs and causal direction, and pairwise causal judgment and causal question answering. The overall picture is ``knowledge with bias''. LLMs can generate causal graphs that agree with domain knowledge to a considerable degree, but are highly sensitive to prompt wording, decoding temperature, and graph scale. The deviation between distilled and true graphs is structural rather than random noise (this paper's M1 gain decomposition and M5 temperature V-shape provide injection-side evidence for this picture). Iterative approaches to improving causal discovery have also appeared: closed-loop routes in which an LLM generates causal hypotheses that are adjudicated by argumentation and empirical data (e.g., COAT-style methods). Because the authorship attribution of this line remains unresolved, we note the direction without citing its details. Our departure from this line is that we do not improve ``graph generation''; we treat the generated artifacts as candidate material and study the selection problem of ``which one to inject, into which base''.

\subsection{Data-side intervention and prior-injection routes}

Three classes of work are closest to our injection form. Curriculum learning organizes training data by difficulty \cite{bengio2009curriculum}; data augmentation expands the diversity of synthetic samples; continuing to train a pretrained model on synthetic data is common in LLM alignment and domain adaptation. In addition, the data selection / data attribution line studies ``which training samples contribute most to a model'', conceptually close to our ``which prior contributes most to a model'', but its object is subset selection from an existing corpus, whereas our object is a \emph{parameterizable, generable task distribution}. The candidate is not ``which data to pick'' but ``which generator to pick''. What these routes share is that they treat ``what data to use'' as a pre-given or heuristically fixed manual decision; we instead cast that decision itself as a closed-loop optimization with cheap-proxy evaluation (Equation~\eqref{eq:selection}), and this is where we differ.

\subsection{Closed-loop search and self-improvement family}
\label{sec:related-search}

Two mature lines of work close the loop on ``selection''. The search line: neural architecture search \cite{elsken2019nas} and hyperparameter optimization \cite{li2018hyperband} have long established the two-stage ``candidate pool + cheap-proxy evaluation + full validation'' paradigm, with cheap-proxy design (early stopping, low fidelity, weight sharing) as the core problem. Our 5k-step pre-screening and self-restraint clause (Section~\ref{sec:twostage}) are a low-fidelity instance of the same idea. The self-improvement line: AREX \cite{arex2026} provides design principles for round-transition protocols (accept/refine/restart) and ``improvement-state'' bookkeeping (we verified its full text and adopted it for our framework's round management). V-STaR \cite{hosseini2024vstar} alternates generator and verifier training and uses failed samples as verifier training material, supporting the idea that ``rejected candidates are not waste''. The asymmetric-verification line \cite{zeng2025asymmetric} discusses compute allocation when verification and generation costs are asymmetric. Our setting, however, runs in the opposite direction: \emph{prior verification (full post-training) is far more expensive than generation (distilling one graph)}, which is why cheap screening is indispensable in the two-stage loop. Our fundamental difference from this family lies in the optimization object: prior work optimizes answers, trajectories, or prompts, whereas we optimize the semantic content of a prior distribution.

\subsection{Baseline spectrum and standard benchmarks}

Our external-validity comparison (Sections~\ref{sec:m8}, \ref{sec:m9}) covers two objects. The first comprises classical per-dataset estimators: the S-/T-/X-learner meta-learner family \cite{kunzel2019metalearners}, the DR-learner \cite{kennedy2023dr}, random-forest-based heterogeneous-effect estimation (the methodological basis of CausalForestDML \cite{athey2018grf}), and a naive ATE baseline estimating only the average effect; these represent the ``non-amortized, per-dataset fit'' end. The second is the causal community's standard evaluation benchmarks: IHDP \cite{hill2011ihdp} (semi-synthetic, with individual-effect ground truth) and Lalonde \cite{lalonde1986evaluating,dehejia1999causal} (real observational data, with a randomized-experiment estimate as reference) are the two most widely used benchmarks in the causal-effect-estimation literature. We place both injected configurations and base versions on these benchmarks to answer whether the gain survives outside our own domain.

\section{Method: Closed-Loop Prior Selection Framework}

The preliminaries are given in Section~\ref{sec:prelim}; Section~\ref{sec:formulation} formalizes the prior-injection problem, Section~\ref{sec:design} presents the framework design, and Section~\ref{sec:theory} gives the theoretical analysis.

\subsection{Preliminaries: structural causal models and amortized inference}
\label{sec:prelim}

\begin{definition}[Structural causal model \cite{pearl2009causality}]
\label{def:scm}
A structural causal model (SCM) is a tuple $\mathcal{M}=\langle U, V, F, P(U)\rangle$: $U$ is a set of exogenous noise variables with distribution $P(U)$; $V=\{V_1,\dots,V_d\}$ is a set of endogenous variables; $F=\{f_1,\dots,f_d\}$ is a set of structural equations specifying how each variable is determined by its direct causes and noise:
\begin{equation*}
V_i := f_i(\mathrm{PA}_i,\, U_i), \qquad i=1,\dots,d.
\end{equation*}
\end{definition}

Intuitively, an SCM uses a set of equations to describe how each variable in a system is ``generated'' by its direct causes; the directed graph formed by these cause--effect dependencies is what is commonly called a causal graph. An SCM induces an observational distribution $P_{\mathcal{M}}(v)$. Intervening on $V_j$ with $\mathrm{do}(V_j=t)$ \cite{pearl2009causality} replaces the $j$-th structural equation with the constant $t$, yielding the post-intervention model $\mathcal{M}_t$ and interventional distribution $P_{\mathcal{M}}(v \mid \mathrm{do}(t))$. The potential outcome of treatment $T$ on outcome $Y$ is $y_t(u) := Y$ evaluated under $\mathcal{M}_t$ with $U=u$ \cite{rubin1974estimating}; the individual and average treatment effects are
\begin{equation*}
\tau(u) := y_1(u) - y_0(u), \qquad \mathrm{ATE} := \E_U\big[\,y_1(U) - y_0(U)\,\big].
\end{equation*}
In plain terms, a potential outcome describes ``the outcome the same individual would have received under treatment versus no treatment''; because the same individual cannot be in both states at once, the individual treatment effect cannot be directly observed, and this is the fundamental difficulty of causal inference. The ATE is its population average and the estimand of this paper. The task of causal-effect estimation is to estimate the ATE (or its conditional version $\mathrm{CATE}(x)$) from observational samples of $P_{\mathcal{M}}$. Going from observation to intervention requires assumptions encoded by the graph (e.g., ignorability); this is what makes the causal graph a valuable knowledge carrier.

\paragraph{Amortized causal inference.}
We formalize a causal-effect-estimation problem as a task $\tau = (\mathcal{M}, D, q)$: an SCM $\mathcal{M}$ generates an observational dataset $D=\{(x_i, t_i, y_i)\}_{i=1}^n$; a query $q=(x_q, t_q)$ asks for the outcome of an individual under intervention, whose ground truth $y_q^*$ is solved from the post-intervention model $\mathcal{M}_{t_q}$. A \emph{task prior} is a distribution $\mathcal{P}$ over tasks:
\begin{equation*}
\tau \sim \mathcal{P}: \quad \mathcal{M}\sim P_{\mathcal{M}},\;\; D\sim P_{\mathcal{M}}^{\,n},\;\; q\sim P_q .
\end{equation*}
PFN-family models parameterize a map $f_\theta$ that conditions on $D$ as context and directly predicts $q$, with training objective
\begin{equation}
\label{eq:pfn}
\mathcal{L}(\theta) \;=\; \E_{\tau\sim\mathcal{P}}\; \E_{(D,\,q,\,y^*)\sim\tau}\; \ell\big(f_\theta(D,q),\, y^*\big).
\end{equation}
In Equation~\eqref{eq:pfn}: $\tau \sim \mathcal{P}$ is a causal task sampled from the prior; $(D, q, y^*)$ are respectively that task's observational data, the query to answer, and the query's ground truth (solved from the SCM under intervention); $f_\theta$ is the model parameterized by $\theta$, which takes $D$ as input context and predicts $q$; $\ell(\cdot,\cdot)$ is a task-level loss measuring the discrepancy between prediction and ground truth. The objective seeks parameters $\theta$ minimizing the model's average loss over all tasks covered by the prior. After training, the model needs no per-task optimization for any new task under the same prior; one forward pass $f_\theta(D,q)$ yields the estimate, and the cost of inference is thereby amortized. The paradigm has one property that is central to this paper: \emph{the capability set of $f_\theta$ is determined by the support of the prior $\mathcal{P}$}, i.e., the model can only reliably handle causal-system structures that appear in $\mathcal{P}$.

\subsection{Problem formulation: the injection operator and amortized prior selection}
\label{sec:formulation}

As noted in Section~\ref{sec:related-llm}, the deviation between distilled and true graphs is structural; their value and risk therefore cannot be inferred from graph-level fidelity and must be measured by injection. We first formalize the process that takes a distilled graph to an injected model.

\begin{definition}[Distilled graph, compilation, injection]
Let $G_L$ be a causal graph distilled by an LLM for a target domain. A compilation pipeline compiles $G_L$ into an executable spec $s$ (an SCM program in the sense of Definition~\ref{def:scm}), from which a task distribution $q_s$ is generated. \emph{Injection} means continuing to train a pretrained base $\theta_0$ on $q_s$:
\begin{equation}
\label{eq:injection}
\mathcal{T}(\theta_0, q) \;:=\; \arg\min_{\theta}\;\; \E_{\tau\sim q}\;\ell\big(f_\theta(\tau)\big), \qquad \text{initialized at } \theta_0.
\end{equation}
\end{definition}

Here $\theta_0$ is the pretrained base's parameters; $q$ is the task distribution generated by some candidate prior; $\ell$ has the same meaning as in Equation~\eqref{eq:pfn}. The injection operator $\mathcal{T}(\theta_0, q)$ denotes the process of ``initializing at $\theta_0$ and continuing to train on the tasks of candidate prior $q$''; its output is the injected model's parameters. Intuitively, injection lets the model keep learning on the tasks generated by a distilled graph, ``absorbing'' the causal structure carried by the graph into its parameters.

The risk of an injected model on an evaluation domain (task distribution $p$) is $R_p(\theta) := \E_{\tau\sim p}\,\ell\big(f_\theta(\tau)\big)$, an error-type loss where lower is better. This gives the core optimization problem of this paper, \emph{amortized prior selection}: given a candidate prior pool $\mathcal{Q}=\{q_1,\dots,q_K\}$, a fixed base $\theta_0$, and a target evaluation domain $p$, solve, within a total budget $B$,
\begin{equation}
\label{eq:selection}
q^*(\theta_0, p) \;=\; \arg\min_{q\,\in\,\mathcal{Q}}\;\; R_p\big(\mathcal{T}(\theta_0,\, q)\big), \qquad \text{s.t.}\;\; \sum_{q}\mathrm{cost}(q)\le B.
\end{equation}
In Equation~\eqref{eq:selection}: $\mathcal{Q}$ is the candidate prior pool; $q^*$ is the prior to be selected; $\mathcal{T}(\theta_0, q^*)$ is the model obtained after injecting that prior (Equation~\eqref{eq:injection}); $R_p(\cdot)$ is the model's risk on the evaluation domain (lower is better, hence $\arg\min$); $\mathrm{cost}(q)$ is the training budget consumed evaluating candidate $q$, and the constraint $\sum_q \mathrm{cost}(q)\le B$ is the evaluation-budget cap. Two scope qualifications for $q^*$ deserve emphasis. First, $q^*$ is explicitly conditioned on the base $\theta_0$ and the evaluation domain $p$: our winner is the optimal prior for a specific base and a specific domain, and its validity is confined to that base and domain. M6 (Section~\ref{sec:m6}) shows the same recipe reverses ranking on another base, and M9 (Section~\ref{sec:m9}) shows it fails once the task leaves the training-prior support; $q^*$ is therefore not a cross-base, cross-domain universal solution. Second, the composite score $\mathrm{Score}(q)$ in Equation~\eqref{eq:score} is \emph{not} an unbiased estimate of $R_p$ but an engineered cheap proxy, whose design rationale follows Theorem~\ref{thm:reversal} (proxy evaluation must be dominated by real-domain transfer). We do not claim that $\mathrm{Score}$ faithfully approximates $R_p$; the final verdict on $R_p$ is given by the second-round full validation (Section~\ref{sec:twostage}). The difficulty of Equation~\eqref{eq:selection} is that evaluating $R_p$ is expensive (full post-training), so cheap proxies are needed for coarse candidate screening; this is the design goal of the framework in Section~\ref{sec:design}.

\subsection{Framework design}
\label{sec:design}

\subsubsection{Empirical premise for the injection channel}

Before building the framework, we ruled out a lighter alternative: inference-time injection. There, the distilled graph is transcribed into a few synthetic demonstration samples and concatenated into the input context, in the hope that the model would follow them at test time. Relative to the uninjected official base (primary-domain NMSE $0.031$), every demonstration-injection configuration collapsed, with the error rising to $0.09$--$0.74$. Demonstrations generated by the domain's true graph collapsed the most ($+0.714$), more than distilled or random graphs. Correctness of knowledge cannot compensate for the mismatch between the injected distribution and the training distribution. This establishes the injection channel as the post-training form of Equation~\eqref{eq:injection} (full data in the appendix).

\subsubsection{Compilation and observational-regime derivation}

\begin{definition}[Three observational regimes]
\label{def:regimes}
Given a spec $s$, three families of training tasks are derived, corresponding to three observational regimes:
\begin{itemize}
  \item \emph{Confounded} (conf): all unobserved-confounding structure is retained and observational data are generated in the presence of common causes; task distribution $q_s^{\mathrm{conf}}$.
  \item \emph{Exogenous intervention} (exo): edges from unobserved variables to the treatment are removed, corresponding to settings where treatment is independent of confounders; $q_s^{\mathrm{exo}}$.
  \item \emph{Partially observed} (part): some latent variables are observed, corresponding to incomplete-information settings; $q_s^{\mathrm{part}}$.
\end{itemize}
\end{definition}

A mixed prior mixes the three families by a ratio $\pi$:
\begin{equation}
\label{eq:mix}
q_s^{\pi} \;:=\; \pi_1\, q_s^{\mathrm{conf}} + \pi_2\, q_s^{\mathrm{exo}} + \pi_3\, q_s^{\mathrm{part}}, \qquad \pi = (0.5,\, 0.3,\, 0.2).
\end{equation}
Here $\pi_1,\pi_2,\pi_3$ are the mixing ratios of the three observational regimes, satisfying $\pi_1+\pi_2+\pi_3=1$; the mixed prior $q_s^{\pi}$ samples training tasks from the three families by these ratios. Training on a single regime alone makes the model overfit the statistical patterns specific to that regime; mixing lets the model see different observational forms of the same causal structure. $\pi$ therefore acts as a dosage knob for prior injection, whose non-monotone effect is empirically demonstrated in Section~\ref{sec:m2}.

\subsubsection{Candidate pool}

The candidate pool expands along two orthogonal dimensions:
\begin{equation*}
\mathcal{Q} \;=\; \big\{\, q_{s,\tau} : s \in S,\;\tau \in \{0.3,\,0.7,\,1.5\} \,\big\} \;\cup\; \big\{ q_{\mathrm{oracle}},\; q_{\mathrm{rand}} \big\},
\end{equation*}
where $S$ is the set of distillation sources (the primary source V4, a 9B-parameter LLM, and an alternative distillation configuration AESD), and $\tau$ is the decoding temperature during distillation (denoted v4, v4t07, v4t15). The two reference configurations are the domain's true graph $q_{\mathrm{oracle}}$ and a semantically shuffled random graph $q_{\mathrm{rand}}$ (used to separate ``content'' from ``diversity''). The temperature dimension carries the mechanism hypothesis: if the gain comes from structural diversity, $q_{\mathrm{rand}}$ should be equivalent to the best configuration; if it comes from semantic content, only distilled graphs carrying correct content help (adjudicated in Section~\ref{sec:m5}).

\subsubsection{Composite score}

Each candidate undergoes cheap post-training $\widetilde{\mathcal{T}}(q)$ (5k steps), and four sub-scores are computed: primary-domain generalization $m_1(q) = -\widehat{R}_{\mathrm{law}}\big(\widetilde{\mathcal{T}}(q)\big)$, diagnostic fit $m_2(q)$ (confounded diagnostic panel), diagnostic generalization $m_3(q)$ (general diagnostic panel), and capability retention $m_4(q)$ (the negative of erosion of the base's existing capability). The composite score is
\begin{equation}
\label{eq:score}
\mathrm{Score}(q) \;=\; \sum_{i=1}^{4} w_i\, m_i(q), \qquad w = (0.50,\; 0.25,\; 0.15,\; 0.10).
\end{equation}
In Equation~\eqref{eq:score}: $\widetilde{\mathcal{T}}(q)$ is the model obtained by cheap post-training (5k steps) of candidate $q$; $m_1(q),\dots,m_4(q)$ are the four sub-scores, all normalized to ``higher is better'', in order primary-domain generalization $m_1$, confounded diagnostic fit $m_2$, diagnostic generalization $m_3$, and capability retention $m_4$; $w_1,\dots,w_4$ are frozen weights summing to 1, fixed once and not adjusted across all experiments. The composite score is a weighted average over several evaluation dimensions, with real-domain generalization weighted above half. \emph{This weighting is not a prior preference but a direct repair motivated by Theorem~\ref{thm:reversal} and the earlier empirical divergence (Sections~\ref{sec:m1}--\ref{sec:m4})}: under equal weights, the framework would systematically pick wrong.

\subsubsection{Two-stage closed loop}
\label{sec:twostage}

The two-stage procedure is given in Algorithm~\ref{alg:loop}. The round-one ranking carries a pre-registered self-restraint clause. If the ranking positions of the reference configurations $\{q_{\mathrm{oracle}}, q_{\mathrm{rand}}\}$ violate the pre-registered expectation, the ranking is demoted to ``exclude clearly poor candidates only'' and no global discriminative power is claimed (this clause was in fact triggered in Section~\ref{sec:m4}). Statistical discipline: all comparisons use $n\ge3$ seeds and paired $t$-tests; significance requires $p<0.05$ together with an effect size passing the pre-registered threshold; standard deviations use the sample convention (ddof=1).

\begin{algorithm}[t]
\caption{Two-stage closed-loop prior selection}
\label{alg:loop}
\SetAlgoLined
\KwIn{candidate pool $\mathcal{Q}$, base $\theta_0$, budget (cheap $B_1 = 5$k steps, full $B_2 = 20$k steps)}
\textbf{Round 1 (cheap screening):}\;
\For{$q \in \mathcal{Q}$}{
  $\widetilde{\theta} \gets$ post-train on $q$ for $B_1$ steps (initialized at $\theta_0$)\;
  compute $\mathrm{Score}(q)$ (Equation~\eqref{eq:score})\;
}
\If{reference-config ranking violates pre-registered expectation}{
  demote ranking to ``exclude clearly poor candidates'' \tcp*{self-restraint clause}
}
take the top candidate set $\mathcal{Q}' \subseteq \mathcal{Q}$\;
\textbf{Round 2 (full validation):}\;
\For{$q \in \mathcal{Q}'$}{
  $\theta \gets$ post-train on $q$ for $B_2$ steps, $n$ seeds (initialized at $\theta_0$)\;
  paired statistical test on $R_p(\theta)$\;
}
\KwOut{winner $q^*$ and a pre-registered verdict report}
\end{algorithm}

\begin{figure}[t]
\centering
\includegraphics[width=\textwidth]{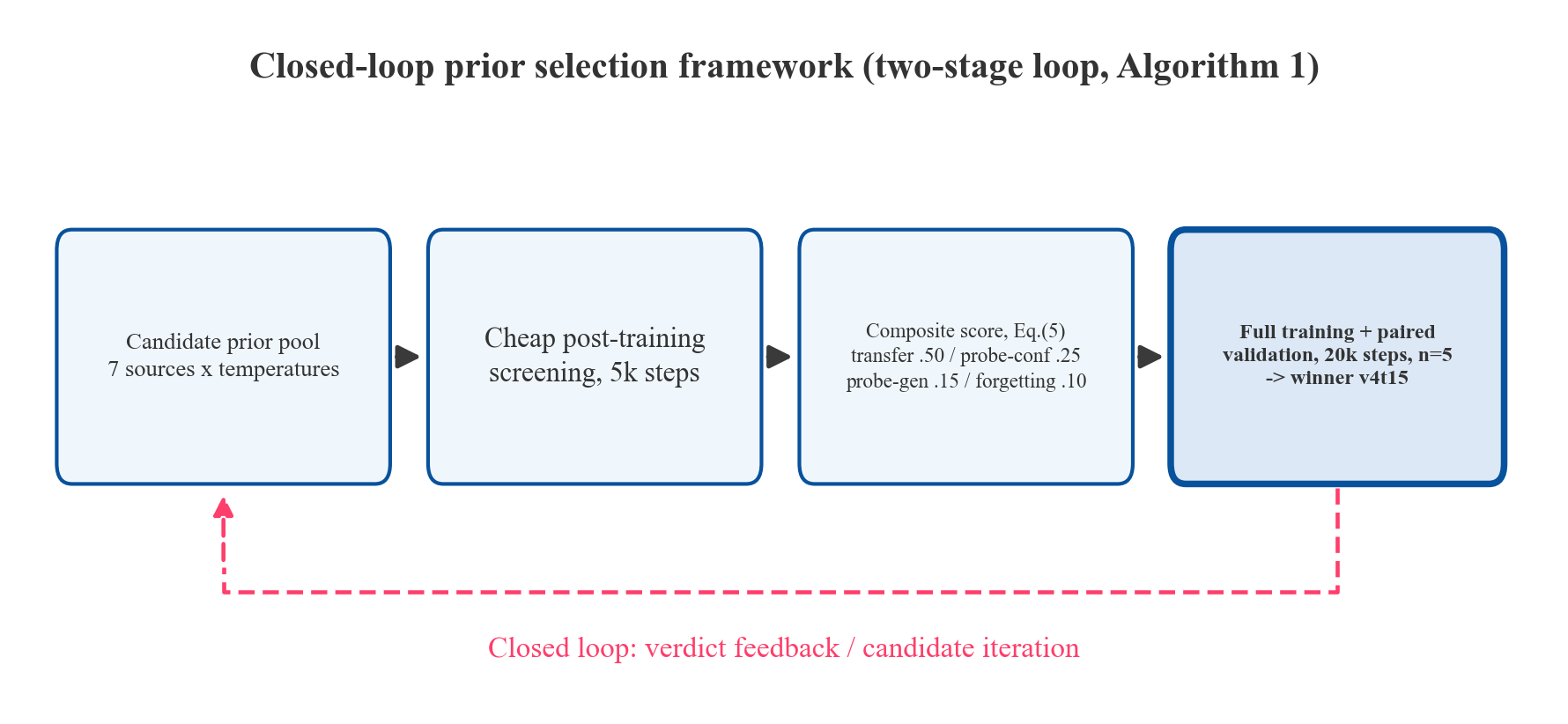}
\caption{Schematic of the closed-loop prior selection framework (two-stage loop, Algorithm~\ref{alg:loop}). Conceptual schematic.}
\label{fig:f1}
\end{figure}

Figure~\ref{fig:f1} sketches the framework.

\FloatBarrier
\subsection{Theoretical analysis}
\label{sec:theory}

This section gives three theorems. Each corresponds to one class of mechanism evidence: the ensemble dilution of Section~\ref{sec:m3}, the diagnosis--generalization divergence of Sections~\ref{sec:m1}, \ref{sec:m3}, and \ref{sec:m4}, and the out-of-support non-repair of Section~\ref{sec:m9}. Each is presented as assumption--statement--proof--remark (experimental anchoring). We note up front that the three-condition empirical regularity of Section~\ref{sec:discussion} is an empirical induction, not a theorem; the theorems here provide conceptual grounding for the selection mechanism and the non-repair boundary, and we make no stronger claim.

\begin{theorem}[Non-vanishing lower bound on ensemble error under correlated errors]
\label{thm:ensemble}
\emph{Assumption.} $K$ distillation sources produce predictors $f_k = f^* + \varepsilon_k$ ($k=1,\dots,K$), where $f^*$ is the target function, the errors satisfy $\E[\varepsilon_k]=b_k$, $\E[\varepsilon_k^2]=\sigma^2$, and there exists $\rho\in(0,1]$ such that for all $j\neq k$, $\E[\varepsilon_j\varepsilon_k]\ge \rho\,\sigma^2$.

\emph{Statement.} The equal-weight ensemble $\bar{f}=\frac{1}{K}\sum_k f_k$ has mean-squared error
\begin{equation*}
\mathrm{MSE}(\bar{f}) \;=\; \frac{\sigma^2}{K} + \frac{K-1}{K}\,\bar{C} \;\ge\; \rho\,\sigma^2, \qquad \bar{C} := \frac{1}{K(K-1)}\sum_{j\neq k}\E[\varepsilon_j\varepsilon_k],
\end{equation*}
so the error has a lower bound that does not vanish as $K$ grows; by contrast, if the errors are pairwise uncorrelated ($\bar{C}=0$), then $\mathrm{MSE}(\bar{f})=\sigma^2/K$.
\end{theorem}

\begin{proof}
Expanding $\E[(\bar{f}-f^*)^2]=\frac{1}{K^2}\sum_j\sum_k\E[\varepsilon_j\varepsilon_k]$, the $K$ diagonal terms each equal $\sigma^2$ and the $K(K-1)$ off-diagonal terms each are at least $\rho\sigma^2$. Hence
\begin{equation*}
\mathrm{MSE}(\bar{f}) \;\ge\; \frac{1}{K^2}\big[K\sigma^2 + K(K-1)\rho\sigma^2\big] \;=\; \frac{\sigma^2}{K} + \frac{K-1}{K}\rho\sigma^2 \;\ge\; \rho\,\sigma^2. \qedhere
\end{equation*}
\end{proof}

\emph{Remark (experimental anchoring).} Variance compression by ensembling presupposes decorrelated errors; the graph errors of different distillation sources for the same domain are typically positively correlated (similar omissions, similar spurious edges), so the positive-correlation term $\bar{C}$ dominates. M3's observation is consistent with this picture: three-source ensembling degrades by $0.111$, all three seeds degrade, and the cross-seed standard deviation rises about $4.7\times$. Ensembling produced no variance compression and instead diluted the only source carrying effective gain. Theorem~\ref{thm:ensemble} only explains ``ensembling cannot improve''; it does not explain degradation below the single-source level, for which quality dilution is a competing explanation (Section~\ref{sec:m3}).

\begin{theorem}[Probe--transfer ranking can fully reverse]
\label{thm:reversal}
\emph{Assumption.} The diagnostic (probe) prior $q$ and evaluation distribution $p$ are both non-degenerate, with risks $R_q(\theta)=\E_{\tau\sim q}\ell(f_\theta(\tau))$ and $R_p(\theta)=\E_{\tau\sim p}\ell(f_\theta(\tau))$.

\emph{Statement.} If $q$ and $p$ concentrate on two families of tasks with disjoint structure, there exists a pair of models $(\theta_A,\theta_B)$ such that
\begin{equation*}
R_q(\theta_A) > R_q(\theta_B) \quad\text{and}\quad R_p(\theta_A) < R_p(\theta_B),
\end{equation*}
i.e., the good/bad ordering on probes fully reverses on the real evaluation. The contrapositive gives an alignment condition: probe ranking is a sufficient statistic for transfer ranking only if, for every candidate pair, the two risk differences have the same sign.
\end{theorem}

\begin{proof}
Construct two tasks $\alpha$, $\beta$ with disjoint structure; let $p$ concentrate on $\alpha$ and $q$ on $\beta$. Let $\theta_A$ fit $\alpha$ perfectly and fail on $\beta$, and $\theta_B$ the opposite. Then $R_q(\theta_A)>R_q(\theta_B)$ while $R_p(\theta_A)<R_p(\theta_B)$.
\end{proof}

\emph{Remark (experimental anchoring).} Synthetic diagnostic metrics are valid selection proxies only when $q$ and $p$ are aligned on task-relevant structure. We observe four divergence instances that instantiate this theorem. In M1, the probe panel ranks the general prior best and the distilled graph degraded, while the transfer panel fully reverses; in M3, probes show no difference while transfer differs; in M4 round one, the true graph does not reach the top two and the random graph overtakes it, triggering the self-restraint clause; and the random graph's high cheap-screening rank reverses under full training (Section~\ref{sec:m5}). The direct engineering consequence is that real-domain transfer carries weight $0.50$ in Equation~\eqref{eq:score}. Theorem~\ref{thm:reversal}'s conclusion is implemented in the scoring design, not merely cited as a caution.

\begin{theorem}[Risk outside the training support cannot be bounded]
\label{thm:support}
\emph{Assumption.} The training prior $\mathcal{P}$ generates only tasks with feature dimension $d\in[d_{\min},d_{\max}]$; the evaluation task has dimension $d^*>d_{\max}$.

\emph{Statement.} For any model-selection criterion $S$ that takes values only from the training risk (or the model's behavior on the support), there exist two target systems $g_1,g_2$ that induce identical training loss on all tasks with $d\le d_{\max}$ yet whose evaluation risk at dimension $d^*$ can differ arbitrarily; hence $S$ cannot give any upper bound on the model's evaluation risk at dimension $d^*$.
\end{theorem}

\begin{proof}
The expectation of objective \eqref{eq:pfn} takes values only on the support of $\mathcal{P}$. The systems $g_1,g_2$ induce the same loss on the support, so every training-risk-based criterion takes the same value on both; yet their behavior off the support is not constrained by the training objective, so their evaluation risk can differ arbitrarily.
\end{proof}

\emph{Remark (experimental anchoring).} When the evaluation benchmark's input dimension (IHDP, 26 dimensions) far exceeds the prior-support dimension range ($[1,6]$), no prior-level means can repair the evaluation risk. This is a structural failure, not a recipe defect. M9 stage 2 is the empirical counterpart: even with a domain-matched prior (IHDP domain-card distillation), injection still significantly degrades relative to the uninjected base ($\Delta=+1.96$, $p<0.00001$). Theorem~\ref{thm:support} only states that risk outside the support cannot be bounded by the training objective; it does not imply that dimension is the dominant factor in this instance (the attribution is not separated; Section~\ref{sec:m9}). It also points to the repair direction. The only way to remove this boundary is to expand the support itself (retrain the base on a higher-dimensional prior), not to pick a better prior within the existing support. Intuitively, for task types the model has never seen, such as feature dimensions far beyond the training prior, no criterion based on training experience can guarantee its behavior; this is why out-of-support high-dimensional benchmarks cannot be repaired by prior injection.

\paragraph{Relation of the theorems to the three-condition regularity.}
Theorem~\ref{thm:ensemble} explains ``why select rather than combine'', Theorem~\ref{thm:reversal} explains ``what evidence selection must rely on'', and Theorem~\ref{thm:support} explains ``how far selection can repair''; the three provide conceptual grounding for, respectively, the selection mechanism, the evidence discipline, and the non-repair boundary within the three-condition regularity. The three-condition regularity itself (base underfit, prior-domain match, task within support) is an empirical induction from M1--M9, and we do not claim it is a provable proposition.

\section{Setup}
\label{sec:setup}

\paragraph{Model and training.}
The base is Do-PFN \cite{robertson2025dopfns} (7.34M parameters). An earlier platform diagnosis found that the officially released weights perform well but are insensitive to continued training, while a weaker locally retrained base (denoted v2) is plastic; all post-training therefore initializes from v2 and continues for 20k steps (learning rate $5\times10^{-5}$, hyperparameters frozen during the diagnosis stage), while cheap pre-screening uses 5k steps. See the Reproducibility Statement below.

\paragraph{Evaluation.}
The primary evaluation domain is law\_race, a synthetic causal benchmark with a ground-truth graph, where the model's task is to estimate interventional effects from observational data. The primary metric is the normalized mean-squared error (NMSE) on that domain, the final adjudication criterion. The adjacent domain sales monitors collateral damage. Reference anchors: the uninjected official base attains NMSE $0.0308\pm0.0083$ on the primary domain. A set of synthetic diagnostic tasks (probe panel) is also maintained, but per the divergence finding of Section~\ref{sec:discussion}, probes only monitor whether existing capability degrades and are not used as proxies for generalization. External-validity comparisons (M8, M9) additionally place the configurations on our two domains and on the causal community's standard benchmarks (IHDP, Lalonde), compared against classical estimators and CausalPFN under the same protocol. NMSE, PEHE, and ATE relative deviation are all error-type metrics; lower means a more accurate estimate, and we do not restate the direction for each figure.

Nine controlled experiment modules (M1--M9) answer the three questions posed in the introduction. Sections 5.1--5.5 form the mechanism chain: M1 decomposes the gain's origin, M2 examines the observational-regime ratio, M3 examines multi-source ensembling, M4 gives the closed-loop selection main result, and M5 adjudicates the mechanism. Sections 5.6 and 5.7 test generalization through the cross-domain closed loop (M7) and base dependence (M6). Sections 5.8 and 5.9 test external validity: a same-protocol comparison against the SOTA and classical estimators (M8), and two honest boundaries on standard causal benchmarks (M9). All comparisons follow the statistical discipline of Section 4.

The table below maps each module to the question it answers, its key result, and its statistical level. Presentation order follows narrative considerations (mechanism first, then generalization, then external validity) and does not fully match module numbering. Statistical levels follow the discipline of Section 4.

\paragraph{Statistical discipline.}
All comparisons pre-register their protocol; paired designs pair strictly by training seed (or data split); the significance criterion for M1--M8 is $p<0.05$ and $|\Delta|\ge0.02$, while the M9 (IHDP/Lalonde) effect-size threshold is $|\Delta|\ge0.05$, with the two criteria kept separate; comparisons with $n\le3$ are reported as ``directional'' only.

\section{Results}

Nine controlled experiment modules (M1--M9) answer the three questions posed in the introduction. Sections~\ref{sec:m1}--\ref{sec:m5} form the mechanism chain: M1 decomposes the gain's origin, M2 examines the observational-regime ratio, M3 examines multi-source ensembling, M4 gives the closed-loop selection main result, and M5 adjudicates the mechanism. Sections~\ref{sec:m7} and \ref{sec:m6} test generalization through the cross-domain closed loop (M7) and base dependence (M6). Sections~\ref{sec:m8} and \ref{sec:m9} test external validity: a same-protocol comparison against the SOTA and classical estimators (M8), and two honest boundaries on standard causal benchmarks (M9). All comparisons follow the statistical discipline of Section~\ref{sec:setup}.

Table~\ref{tab:modules} maps each module to the question it answers, its key result, and its statistical level. Presentation order follows narrative considerations (mechanism first, then generalization, then external validity) and does not fully match module numbering. Statistical levels follow the discipline of Section~\ref{sec:setup}.

\begin{table}[t]
\centering
\caption{Module map: the question each experiment answers, the key result, and the statistical level.}
\label{tab:modules}
\small
\begin{tabular}{@{}p{1.55cm}p{3.6cm}p{5.2cm}p{3.6cm}@{}}
\toprule
Module & Question answered & Key result & Statistical level \\
\midrule
M1 (\S\ref{sec:m1}) & Content or narrowing? & $\sim$72\% narrowing, $\sim$28\% content & Content share directional ($n=3$, $p=0.071$) \\
M2 (\S\ref{sec:m2}) & Can gain and collateral cost be decoupled? & Non-monotone interior optimum; adjacent domain degrades monotonically & $n=3$ descriptive \\
M3 (\S\ref{sec:m3}) & Can multiple sources be combined? & Ensembling judged negative, degradation of $0.111$ & Directional ($p=0.171$) \\
M4 (\S\ref{sec:m4}) & Which one to inject; can it be automated? & Winner $2.75\times$ gain, three improvements & Formally significant ($p=0.0086$, $p=0.0109$, $n=5$) \\
M5 (\S\ref{sec:m5}) & Why does the gain occur? & Content beats diversity, V-shaped temperature & Directional ($n=3$) \\
M7 (\S\ref{sec:m7}) & Does it transfer to a new domain? & sales from-scratch $3.0\times$ gain, does not reach official & Directional ($n=3$, unpaired, $p=0.0001$) \\
M6 (\S\ref{sec:m6}) & Does the gain depend on the base? & Already-strong base degrades $4$--$9\times$, ranking reverses & Directional ($n=3$, $p=0.0281$) \\
M8 (\S\ref{sec:m8}) & Against SOTA and classical baselines? & Primary domain wins all ($4.1\times$/$6.1\times$ over CausalPFN); secondary domain boundary & Per-domain, $p\le0.006$ \\
M9 (\S\ref{sec:m9}) & Boundary on standard benchmarks? & Two honest boundaries; attribution not separated & Stage 1 significant; stage 2 significant degradation \\
\bottomrule
\end{tabular}
\end{table}

\FloatBarrier
\subsection{Where the gain comes from: about 72\% from distribution narrowing, only 28\% from knowledge content (M1)}
\label{sec:m1}

Five configurations ranked by primary-domain NMSE: LLM-distilled graph v4 $0.068\pm0.020$, true graph oracle $0.078\pm0.011$, random graph rand $0.111\pm0.013$, alternative distillation configuration AESD $0.153\pm0.041$, general-prior control $0.223\pm0.016$. Replacing the general prior with the random graph (also a graph without knowledge content; the only change is that the training distribution narrows) explains $0.112$ of improvement, about 72\% of the total gain. Replacing the random graph with the LLM-distilled graph (introducing knowledge content) improves only a further $0.043$, about 28\% (directional at $n=3$, $p=0.071$). The distilled graph v4 is statistically indistinguishable from the true graph oracle, and v4 also attains the best adjacent-domain value across the board ($0.166$). M1 also exposes the first diagnosis--generalization divergence. The synthetic diagnostic panel ranks the general prior best and the distilled graph degraded, while the transfer panel fully reverses (the first instance of the misalignment proved by Theorem~\ref{thm:reversal}, see Section~\ref{sec:theory}).

\begin{figure}[t]
\centering
\includegraphics[width=0.82\textwidth]{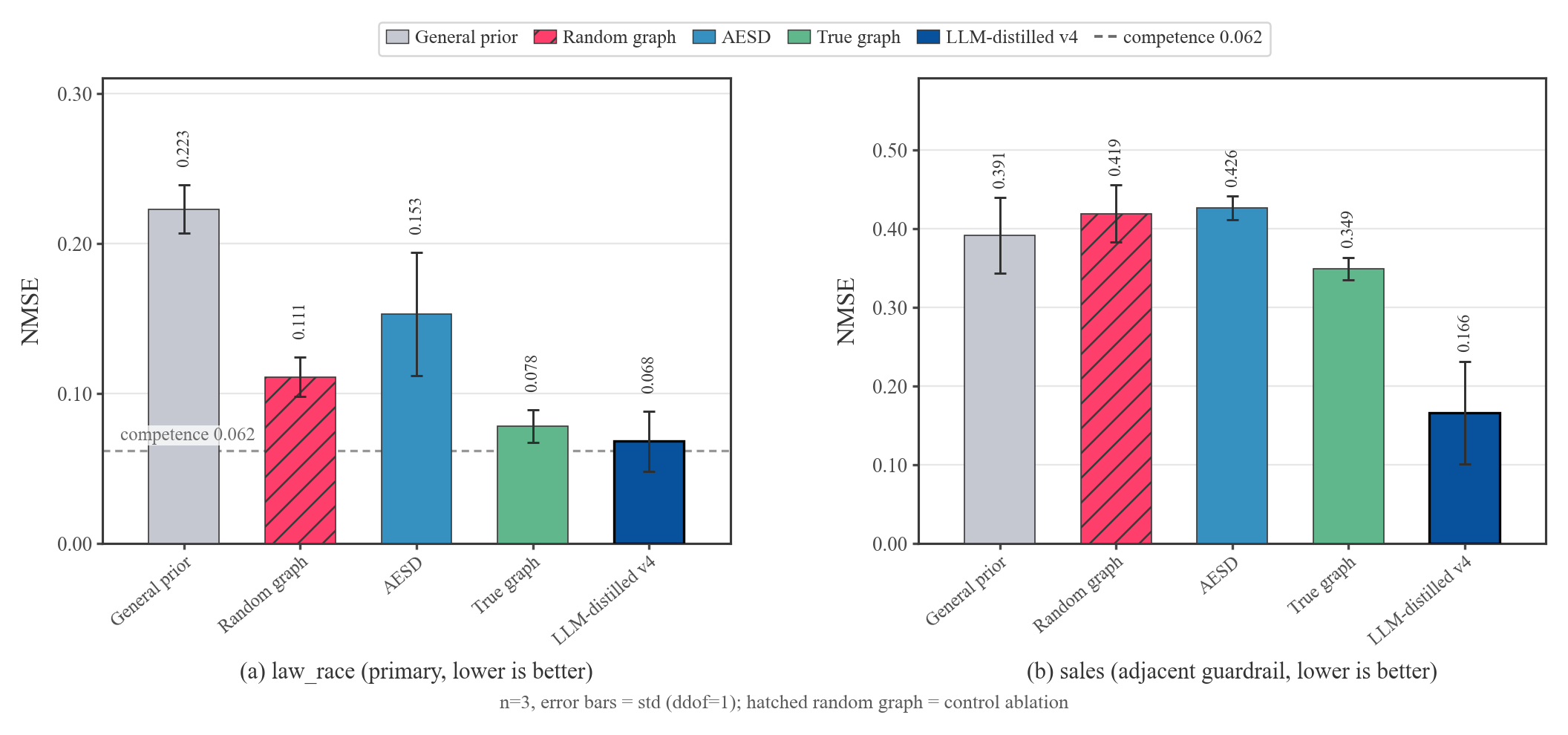}
\caption{M1 single-source injection gain: primary-domain NMSE of the five configurations ($n=3$, error bars $\pm1$ sample std); dashed line is the pre-registered competence threshold $0.062$. The distilled graph is statistically indistinguishable from the true graph and better than the general prior and random graph (directional at $n=3$).}
\label{fig:f2}
\end{figure}

\FloatBarrier
\subsection{Observational-regime ratio: an interior optimum, but the adjacent domain pays a cost (M2)}
\label{sec:m2}

A sweep along the ``confounded-task share'' axis yields a clear shape. Pure confounded primary-domain NMSE is $0.068\pm0.020$; the 50/30/20 mix drops to $0.052\pm0.016$, first crossing the pre-registered competence threshold $0.062$ (the threshold is about twice the official base's primary-domain error; an NMSE at or below it counts as competent); further dilution to 33/33/34 rebounds to $0.100\pm0.015$. The curve is a non-monotone interior optimum. The same axis, however, exposes a collateral cost. Adjacent-domain sales NMSE degrades monotonically along this axis ($0.166 \to 0.387 \to 0.414$). Gain and collateral cost cannot be attained simultaneously by a single configuration, and this tension is the direct motivation for closed-loop selection.

\begin{figure}[t]
\centering
\includegraphics[width=0.66\textwidth]{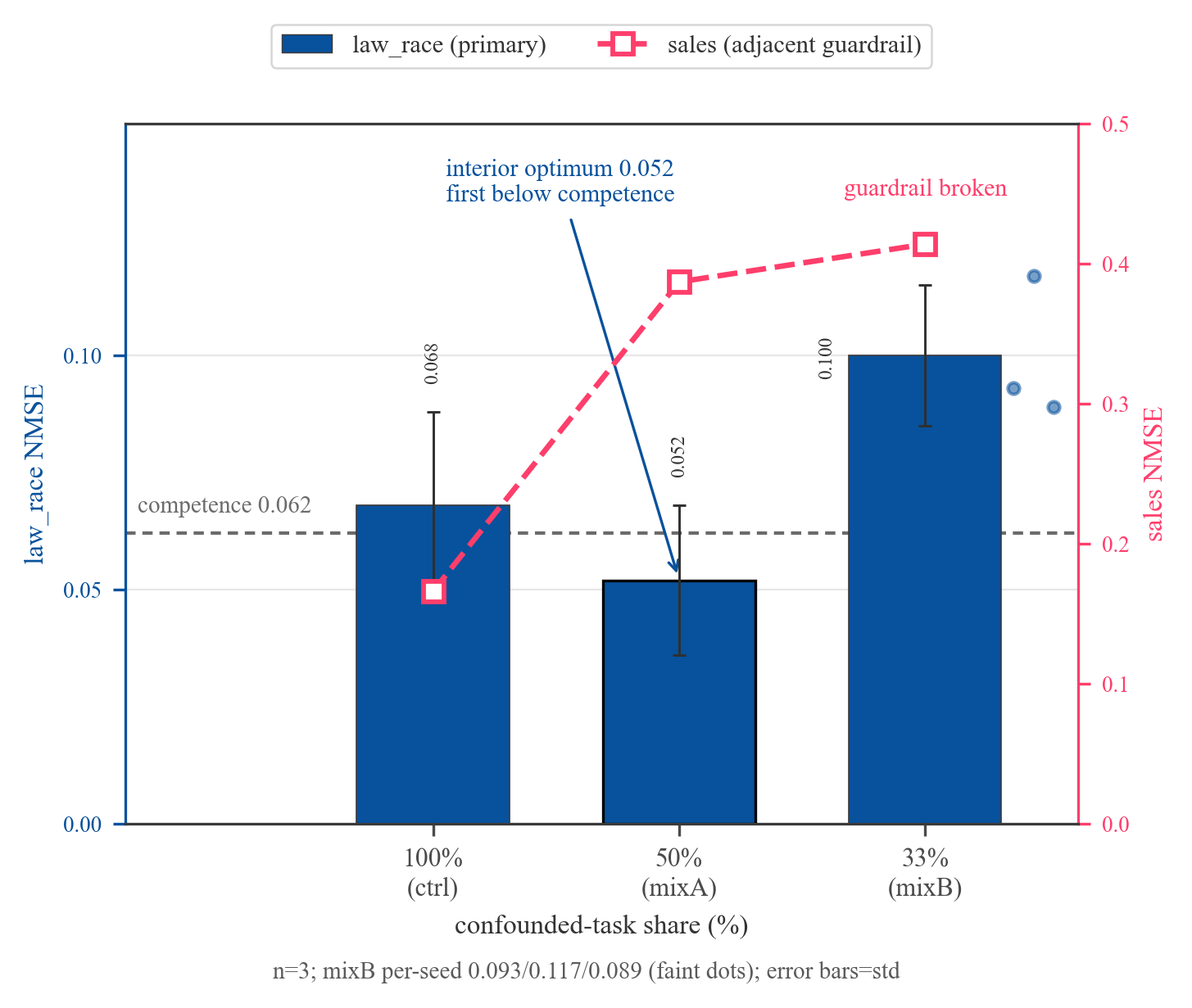}
\caption{M2 observational-regime ratio curve ($n=3$): the primary domain shows a non-monotone interior optimum (50/30/20 mix best, first crossing the competence threshold $0.062$); the adjacent domain degrades monotonically along the same axis. Both axes are NMSE.}
\label{fig:f3}
\end{figure}

\FloatBarrier
\subsection{Multi-source ensembling judged negative (M3, directional evidence)}
\label{sec:m3}

The control is V4 single source + 50/30/20 mix (primary-domain NMSE $0.052\pm0.016$). The ensemble mixes three distillation sources, giving $0.163\pm0.076$, a degradation of $0.111$ (directional negative, $p=0.171$), with all three seeds degrading. Two accompanying pieces of evidence: the cross-seed standard deviation rises about $4.7\times$ (a variance ratio of about $22\times$), again showing no variance compression from ensembling; and the synthetic diagnostic metric shows no difference between the two groups ($p=0.70$), diagnosis again diverging from real transfer. As for the mechanism of degradation, we give two not-yet-separated explanations. Explanation one (positively correlated errors): different LLMs' graph errors for the same domain are similar (similar omissions, similar spurious edges), and averaging cannot cancel positively correlated errors. Theorem~\ref{thm:ensemble} gives a lower bound on ensemble error under equal-variance, positive-correlation assumptions, but that lower bound can only explain ``ensembling cannot improve'', not the degradation below the single-source level observed here. Explanation two (quality dilution): the other two distillation sources aesd ($0.153$) and 9b ($0.135$) are far inferior in quality to V4 ($0.068$), and equal-weight averaging may be dragged down by the inferior sources regardless of whether the errors are correlated. We did not directly measure inter-source error correlation, nor set up a ``same quality, only correlation differs'' control, so the two explanations cannot be separated. We treat only ``ensembling judged negative'' itself as a directional conclusion and list the mechanism attribution as a working hypothesis.

\FloatBarrier
\subsection{Closed-loop selection: main result and three improvements (M4)}
\label{sec:m4}

In round-one cheap pre-screening, the composite-score ranking of the seven candidates triggered the self-restraint clause (true graph not in the top two, random graph overtaking), and the ranking was demoted to ``exclude clearly poor candidates''. This demotion was later shown by mechanism experiments to be well-founded. The round-two main result ($n=5$ paired seeds, full 20k steps): the winner configuration (higher-temperature distilled graph + 50/30/20 mix) attains primary-domain NMSE $0.0211\pm0.0048$, about one third of the baseline $0.0580\pm0.0143$ (a $2.75\times$ gain; paired $t=4.80$, \emph{$p=0.0086$, formally significant}). Its error level also falls below the uninjected official base ($0.0308\pm0.0083$). This is a descriptive cross-lineage comparison: the winner is built on the locally retrained base, the official base is not part of the paired test, and the same recipe degrades when injected into the official base (see Section~\ref{sec:m6}). Three improvements hold simultaneously: a significant primary-domain gain; adjacent-domain sales not broken but improved ($0.356\pm0.020$ vs $0.389\pm0.013$, paired $t=4.49$, \emph{$p=0.0109$, formally significant}, direction 5/5); and no degradation signal on the diagnostic panel (descriptive). The round-one cheap scores agree with the full-training ranking on the top candidates, which supports the extrapolability of pre-screening for these candidates.

\begin{figure}[t]
\centering
\includegraphics[width=0.66\textwidth]{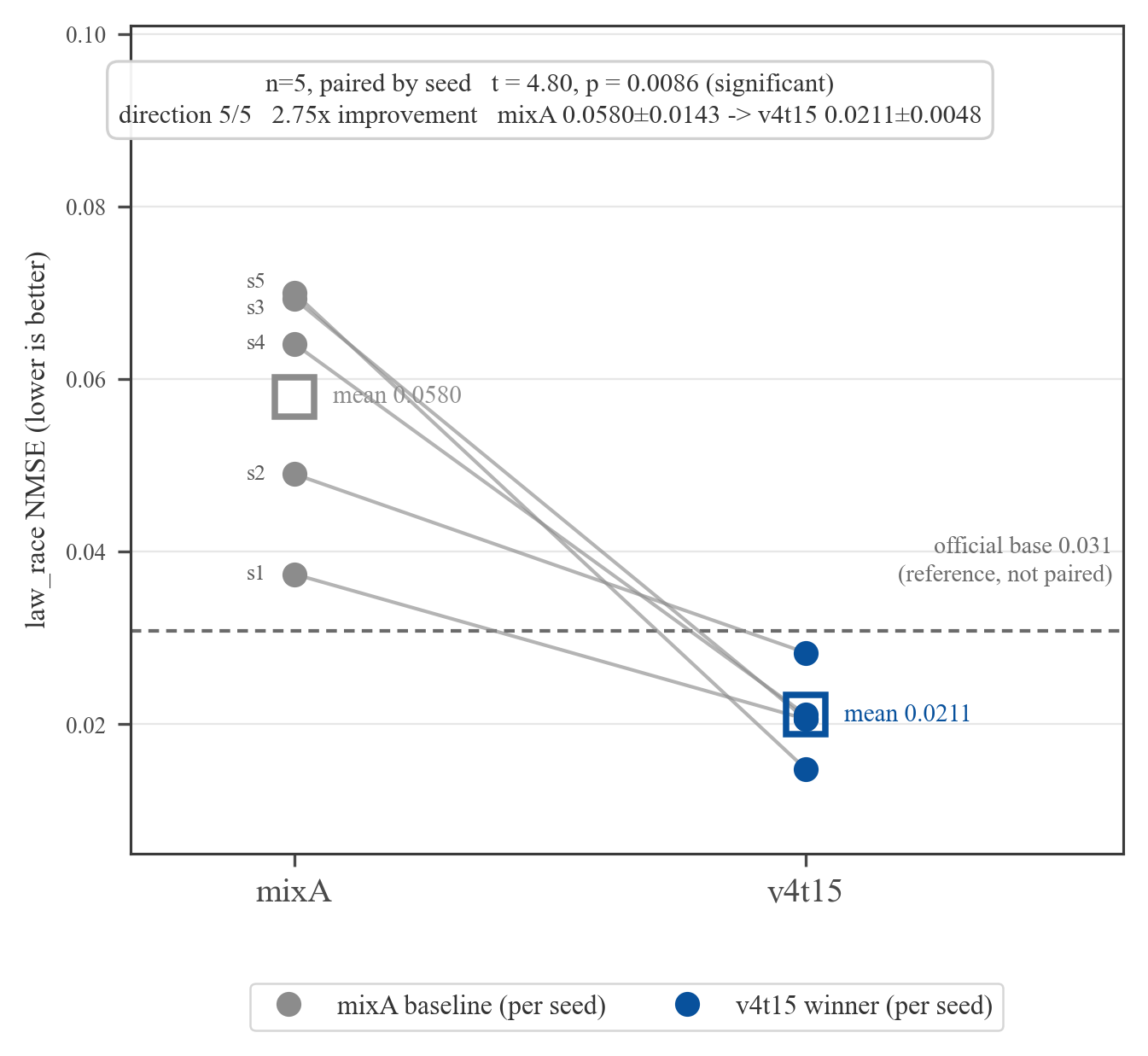}
\caption{M4 main result: mixA and the winner configuration paired strictly by seed ($n=5$, seed $1\leftrightarrow1$ through $5\leftrightarrow5$), five seed-connection lines all sloping downward; squares are means. The horizontal dashed line is the uninjected official base $0.0308$, reference only and not part of the paired test. Paired $t=4.80$, $p=0.0086$, a $2.75\times$ gain, direction 5/5.}
\label{fig:f5}
\end{figure}

\begin{table}[t]
\centering
\caption{Full set of main-result metrics ($n=5$, $\pm1$ sample std; significance criterion $p<0.05$ and $|\Delta|\ge0.02$).}
\label{tab:main}
\small
\begin{tabular}{@{}p{2.8cm}p{2.7cm}p{4.5cm}p{4.5cm}@{}}
\toprule
Dimension & Winner v4t15 & Control & Conclusion \\
\midrule
law\_race (primary) & $0.0211\pm0.0048$ & mixA $0.0580\pm0.0143$; official base $0.0308\pm0.0083$ & $2.75\times$ gain, formally significant ($p=0.0086$); error level below the official base (descriptive cross-lineage, see text) \\
sales (adjacent monitor) & $0.356\pm0.020$ & mixA $0.389\pm0.013$ & Not broken but improved, formally significant ($p=0.0109$, direction 5/5) \\
Diagnostic panel (general) & $0.0092$--$0.0120$ & Base level $\sim0.010$ & No observed forgetting \\
Cross-seed std & $0.0048$ & $0.0143$ & About $3\times$ narrower \\
\bottomrule
\end{tabular}
\end{table}

Figure~\ref{fig:f4} shows the cheap-screening-versus-full-training comparison for the three candidates that have paired data.

\begin{figure}[t]
\centering
\includegraphics[width=0.72\textwidth]{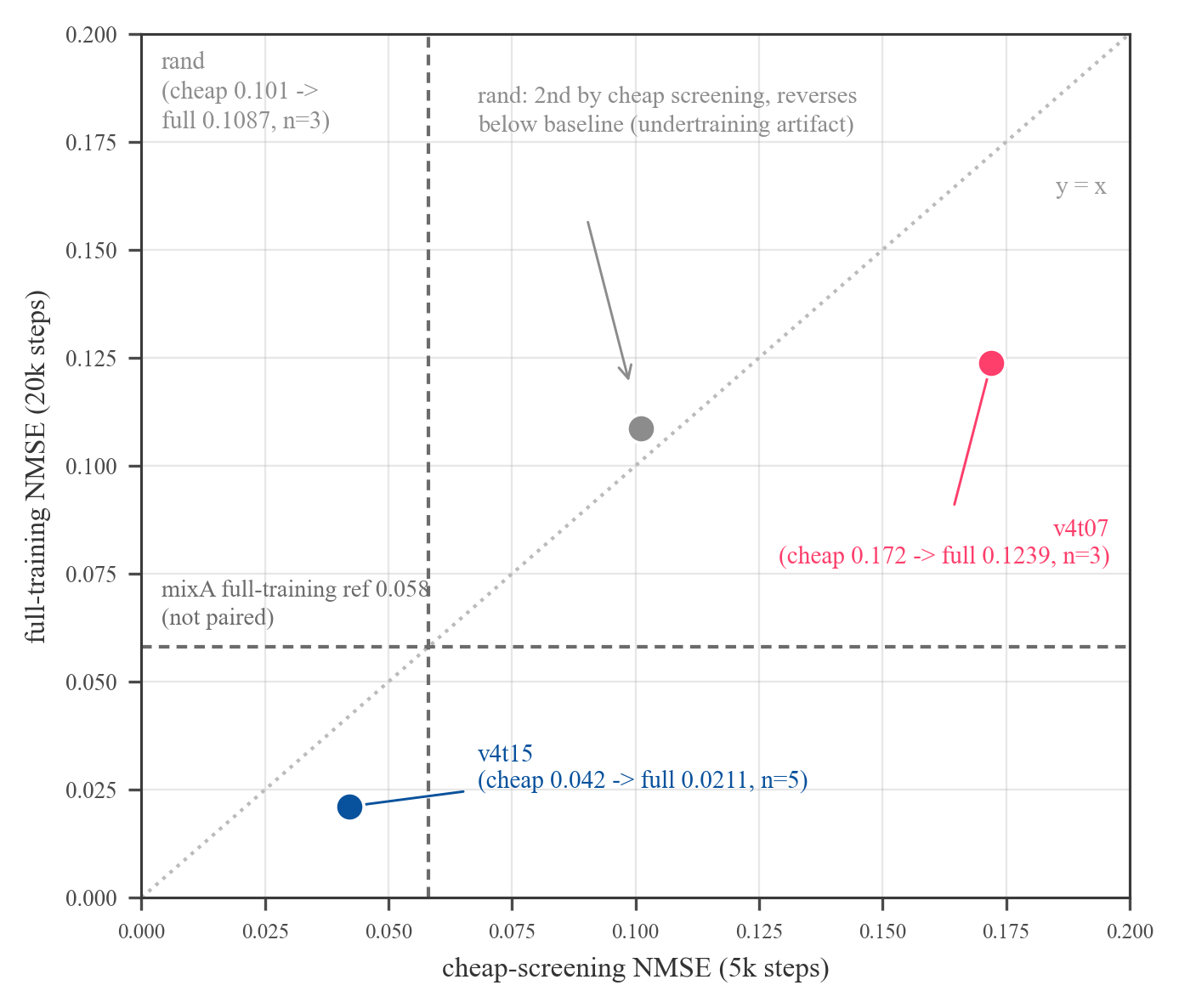}
\caption{Cheap screening (5k-step NMSE) versus full training (20k-step NMSE) for the three candidates with paired data; dashed lines mark the mixA full-training reference $0.0580$; the dotted line is $y=x$. The random graph ranks second by cheap screening yet reverses under full training to below the baseline.}
\label{fig:f4}
\end{figure}

\FloatBarrier
\subsection{Mechanism: content beats diversity (M5)}
\label{sec:m5}

The temperature axis is V-shaped. At $T=0.3$ the primary-domain NMSE is $0.052$; $T=0.7$ degrades to $0.124$ (3/3 seeds agree, $p=0.030$, directional under the $n=3$ discipline, see Section~\ref{sec:setup}); $T=1.5$ is best at $0.023$. The medium-temperature distilled graph undergoes edge-drop narrowing: under medium uncertainty the model tends toward conservative output and loses the edge structure that carries the gain. The random graph has the same order of structural diversity as the best configuration but shuffled semantics. Its primary-domain NMSE is $0.109\pm0.025$, $2.1\times$ worse than the baseline, with 3/3 seeds agreeing in degradation (directional, $p=0.093$), rejecting the ``gain comes from pure diversity'' hypothesis. The direct paired comparison of the random graph against the best configuration (data available) is not reported here, and ``the same order of structural diversity'' lacks an operational measure, so this mechanism adjudication rests on directional evidence. In cheap pre-screening the random graph ranked second, yet under full training it degrades below the baseline, showing that the diversity advantage in the cheap-screening stage does not extrapolate.

\begin{figure}[t]
\centering
\includegraphics[width=0.7\textwidth]{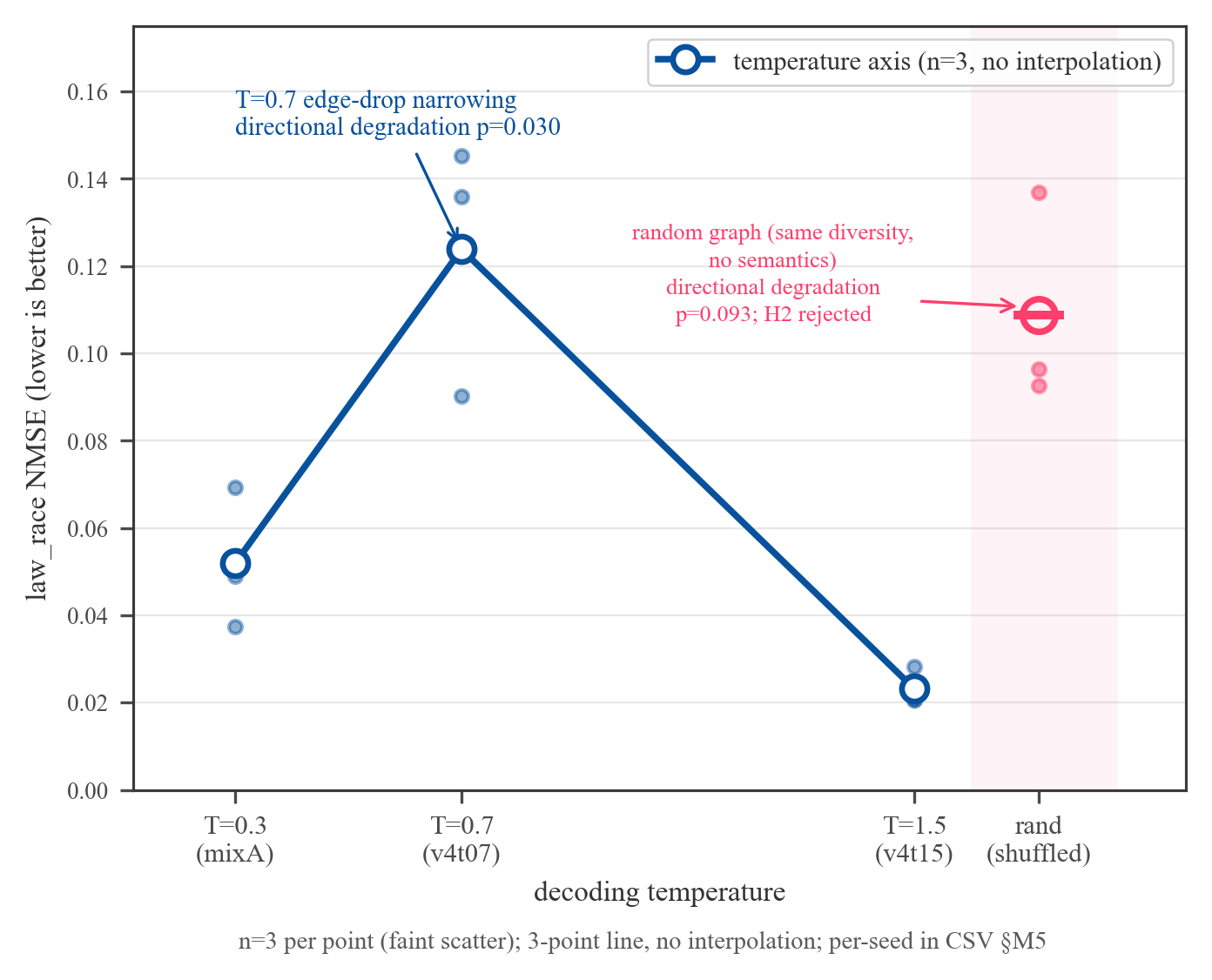}
\caption{M5 mechanism: temperature-axis three-point line ($n=3$, no interpolation); $T=0.7$ degrades due to edge-drop narrowing ($p=0.030$, directional under the $n=3$ discipline); the hollow point is the random-graph control (same structural diversity, shuffled semantics), $2.1\times$ worse than the baseline (directional, $p=0.093$), rejecting the ``pure diversity'' hypothesis.}
\label{fig:f6}
\end{figure}

\FloatBarrier
\subsection{Cross-domain closed loop: from-scratch distillation works on a new domain (M7, positive)}
\label{sec:m7}

To test whether the pipeline transfers to a new domain, we distilled a causal graph from scratch on the sales domain, compiled it, and injected it into the v2 base. Sales-domain NMSE drops from the control (the law\_race prior evaluated on sales) $0.3872$ to $0.1292$, \emph{a $3.0\times$ gain (unpaired $t=-15.21$, $p=0.0001$; directional at $n=3$)}. Cross-domain seeds do not form a legitimate pairing, so an unpaired test is used. It does not reach the official base's $0.0336$ on sales (recorded honestly). One unexpected finding is that the sales-domain prior conversely improves the law\_race primary domain (from the v2 anchor $0.221$ to $0.1103$, 3/3 agreeing).

\begin{figure}[t]
\centering
\includegraphics[width=0.68\textwidth]{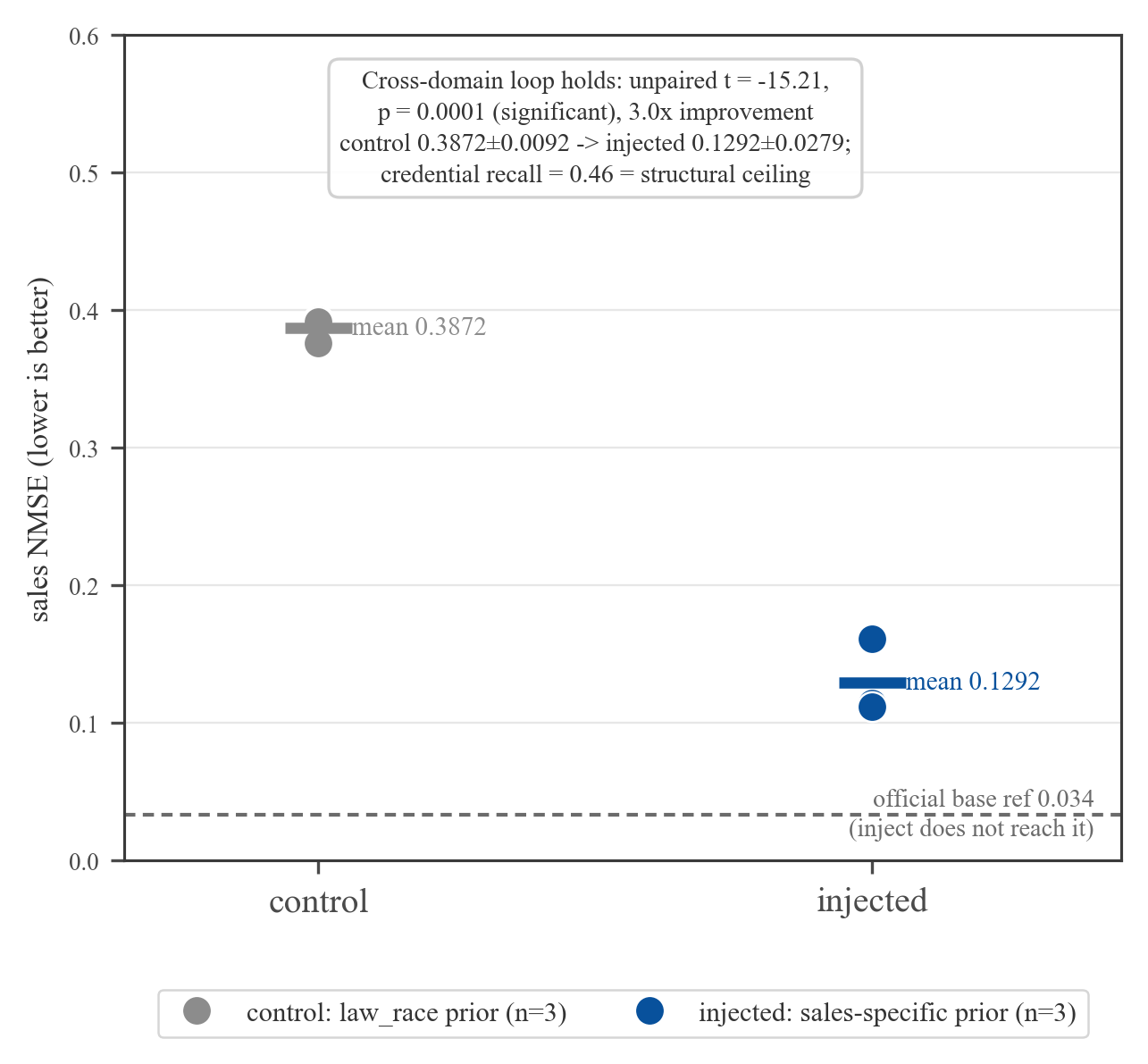}
\caption{M7 cross-domain closed loop: sales-specific prior injection ($n=3$) versus the law\_race prior ($n=3$), per-seed dots plus means; the horizontal dashed line is the official base's reference value $0.0336$ on sales (not reached, recorded honestly). Unpaired $t=-15.21$, $p=0.0001$, a $3.0\times$ gain (directional at $n=3$).}
\label{fig:f9}
\end{figure}

\FloatBarrier
\subsection{Base dependence: the same recipe degrades an already-strong base (M6, negative; directional at $n=3$)}
\label{sec:m6}

We injected the recipe that performed best on the primary domain into the uninjected but already-strong official base. Primary-domain NMSE degrades from $0.0308$ to $0.279$ (higher-temperature configuration) and $0.116$ (mix baseline), a $4$--$9\times$ degradation; on both configurations the adjacent domain also degrades below the official reference. More critically, \emph{the ranking fully reverses}. On the v2 base the higher-temperature configuration is significantly better than the mix baseline (M4, $p=0.0086$), but on the official base it is directionally worse (paired $\Delta=+0.1634$, $t=5.84$, $p=0.0281$, direction 0/3; directional under the $n=3$ discipline, see Section~\ref{sec:setup}). The preliminary conclusion is that, under our frozen recipe, the direction of injection gain is opposite to the base's native level on that domain, producing a significant gain on the underfit v2 base (M4, $n=5$) and a directional degradation on the already-strong official base (M6, $n=3$). We emphasize that this contrast spans only two base points, which is insufficient to support a quantitative monotone relation. We did not re-search the continued-training protocol for the official base, so the degradation may partly come from the mismatch between the recipe and the base state. The more conservative statement is ``under our frozen recipe, injection does not transfer across bases'' (limitations in Section~\ref{sec:discussion}).

\begin{figure}[t]
\centering
\includegraphics[width=0.7\textwidth]{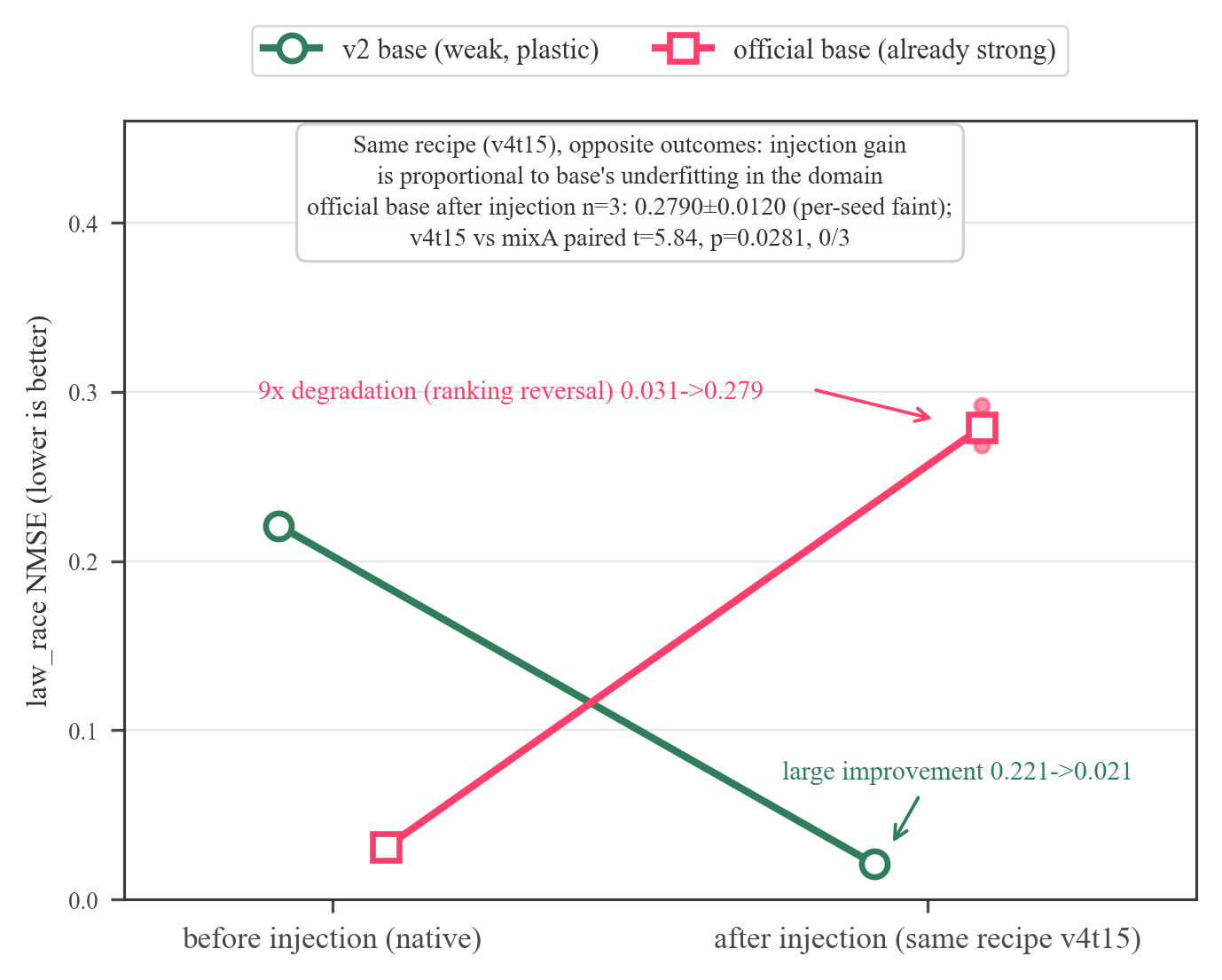}
\caption{M6 base dependence: the same recipe injected into two bases, before/after comparison ($n=3$). The v2 base improves substantially ($0.221\to0.021$); the official base degrades about $9\times$ ($0.0308\to0.279$), and the ranking of the two configurations fully reverses relative to M4.}
\label{fig:f8}
\end{figure}

\FloatBarrier
\subsection{External validity: primary domain wins against the SOTA, secondary domain yields to simple estimators (M8)}
\label{sec:m8}

On our two domains, we compare against the direct SOTA (CausalPFN \cite{balazadeh2025causalpfn}) and a classical-estimator panel (S/T/X/DR-learner \cite{kunzel2019metalearners,kennedy2023dr}, CausalForestDML \cite{athey2018grf}, naive ATE) under the same 5-fold split and the same protocol as model training. \emph{On law\_race, both Do-PFN versions lead all seven machine-learning baselines in point estimate}: the higher-temperature configuration $0.0211$, the official $0.0308$, leading CausalPFN $0.1276$ by $4.1\times$/$6.1\times$ ($p\le0.006$). The only baseline without a significant difference is the naive ATE constant baseline ($0.0278$), a methodological caution: this benchmark has limited individual-CATE heterogeneity. \emph{On sales, simple meta-learners ($\le0.018$) beat all amortized models} (official $0.0336$, CausalPFN $0.1019$, higher-temperature configuration $0.3562$). The $0.3562$ here is the law\_race prior's cross-domain guardrail value measured on sales, strictly distinct from the sales-specific prior of Section~\ref{sec:m7} ($0.1292$). The inductive bias of amortized pretraining wins where the task distribution matches the pretraining prior, and yields to per-dataset fitting on very-small-sample domains where the effect is near-constant.

\begin{table}[t]
\centering
\caption{M8 full comparison on our two domains ($n=5$ fixed splits, paired by split; NMSE, lower is better). Rows are sorted by law\_race NMSE. Per-pair $p$-values are paired $t$-tests of each Do-PFN arm against the named baseline.}
\label{tab:t5}
\small
\begin{tabular}{@{}lccp{5.9cm}@{}}
\toprule
Method & law\_race NMSE & sales NMSE & Outcome vs.\ baselines \\
\midrule
Do-PFN v4t15 (injected) & $\mathbf{0.0211}$ & $0.3562$ & On law\_race, leads all seven baselines in point estimate; significantly better than all machine-learning baselines ($p\le0.0060$ vs CausalPFN; $p\le0.0017$ vs S/T/X/DR-learner and CausalForestDML); no significant difference from naive ATE. On sales, last; significantly worse than all six baselines ($p\le0.030$). \\
Do-PFN official & $0.0308$ & $0.0336$ & On law\_race, second; significantly better than CausalPFN ($p=0.0028$) and all four meta/forest baselines; no significant difference from naive ATE ($p=0.569$). On sales, fifth; directional losses to the top four ($p\in[0.051,0.079]$); significantly better than S-learner only ($p=0.022$). \\
CausalPFN & $0.1276\pm0.0339$ & $0.1019\pm0.0671$ & Direct SOTA. \\
naive ATE & $0.0278$ & $\mathbf{0.0103}$ & Constant-effect baseline; top on sales (effect near-constant). \\
DR-learner & $0.7311$ & $0.0103$ & \\
T-learner & $0.2708$ & $0.0109$ & \\
X-learner & $0.2397$ & $0.0179$ & \\
S-learner & $0.2336$ & $0.3262$ & \\
CausalForestDML & $0.2766$ & $0.1050$ & \\
\bottomrule
\end{tabular}
\end{table}

\begin{figure}[t]
\centering
\includegraphics[width=\textwidth]{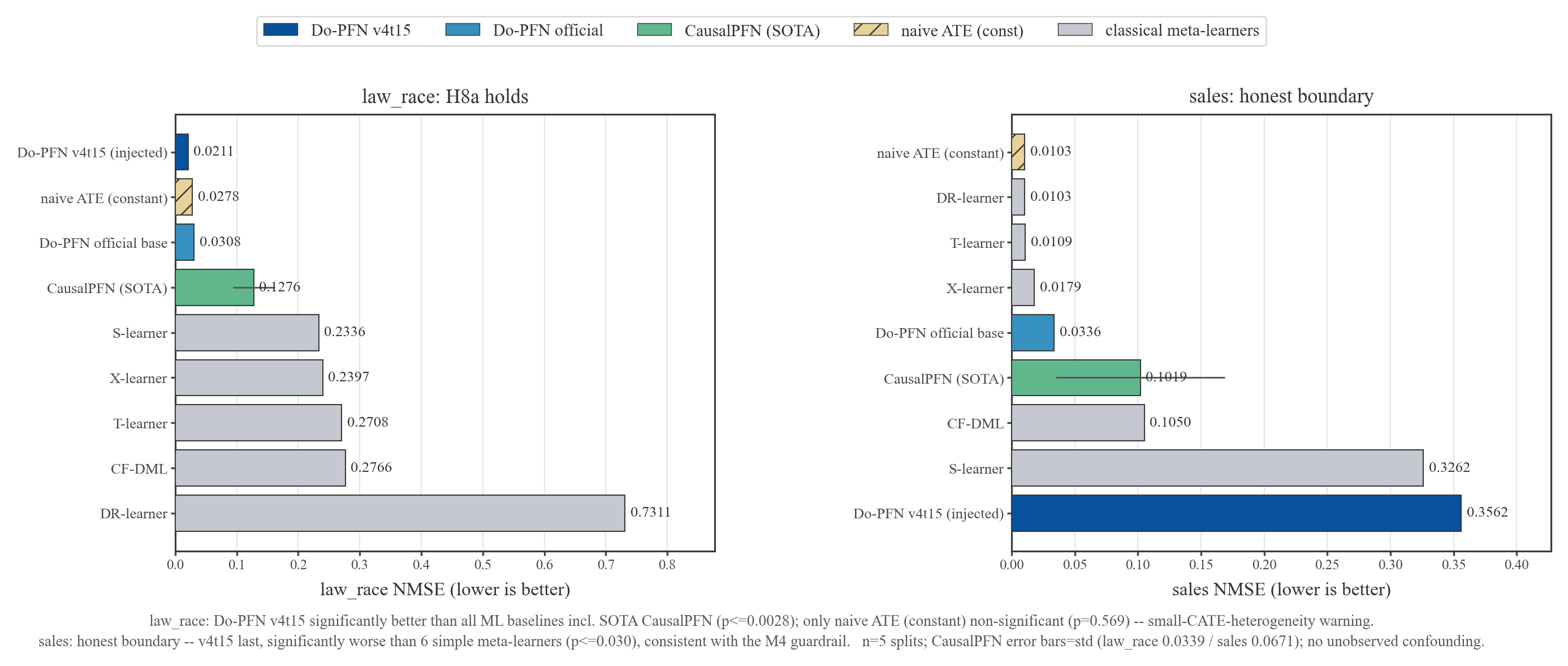}
\caption{M8 SOTA comparison (same 5-fold protocol): left panel, on law\_race both Do-PFN versions lead all machine-learning baselines in point estimate ($4.1\times$/$6.1\times$ over CausalPFN, $p\le0.006$); right panel, on sales simple meta-learners beat all amortized models. naive ATE is the constant baseline. Per-pair $p$-values in Table~\ref{tab:t5}.}
\label{fig:f10}
\end{figure}

\FloatBarrier
\subsection{Standard-benchmark boundary: two honest boundaries (M9)}
\label{sec:m9}

Finally, we push the evaluation to the causal community's standard benchmarks: IHDP \cite{hill2011ihdp} (semi-synthetic, with ground truth, primary metric PEHE, $n=50$ realization pairing) and Lalonde \cite{lalonde1986evaluating,dehejia1999causal} (real observational data + experimental-benchmark ATE \$1794.34, primary metric ATE relative deviation, $n=5$ folds). The effect-size threshold is pre-registered at $|\Delta|\ge0.05$.

\paragraph{Stage 1, pure evaluation: cross-domain negative transfer (first honest boundary).}
On IHDP, CausalPFN is in a class of its own (PEHE $0.7444\pm0.9289$), classical meta-learners $3.26$--$4.27$, while the three Do-PFN versions fall in $4.87$--$6.96$, of which the injected configuration (law\_race domain-specific prior) is bottom at $6.9645$, significantly worse than all 7 baselines. The official base $6.4402$ is also significantly worse than all baselines. On Lalonde, the two uninjected bases are on par with classical methods (no significant difference), while the injected configuration $1.2673$ is significantly worse than 6/7 baselines other than naive ATE. \emph{The cross-domain cost of the domain-specific prior reproduces consistently on both domains}: the higher-temperature configuration, relative to the uninjected base, degrades significantly on IHDP ($\Delta=+2.09$, $p<0.0001$) and on Lalonde ($\Delta=+0.49$, $p=0.0016$). Reported honestly, one unexpected finding is that the retrained v2 base is significantly stronger than the official base on IHDP ($\Delta=-1.57$, $p=0.00037$).

\paragraph{Stage 2, domain-specific distillation: still does not repair the out-of-support high-dimensional benchmark (second honest boundary).}
Does replacing the prior with the target domain's own distilled graph repair it? We distilled on the IHDP domain card and injected, 5-seed final protocol. After injection, IHDP PEHE is $6.8340\pm9.7329$, \emph{significantly degrading relative to the uninjected base ($\Delta=+1.96$, $p<0.00001$, repair failed)}. Relative to the cross-domain prior there is only a weak directional improvement, not reaching significance ($\Delta=-0.13$, $p=0.076$, not reversed after adding seeds). Attribution (not separated): the IHDP benchmark's 26 dimensions far exceed the prior-support dimension range $[1,6]$, an out-of-support input. We list out-of-support dimensionality as the leading candidate explanation, but this attribution has not been confirmed by a separation experiment. Competing explanations include the prior compiled from the IHDP domain card being itself low-quality, and 20k steps of continued training being insufficient. The v2 base is also significantly stronger than the official base on IHDP (see Stage 1), which shows that factors beyond dimension play a substantial role. We have not reconciled this difference. Theorem~\ref{thm:support} only states that risk outside the training support cannot be bounded by the training objective, and does not imply that dimension is the dominant factor in this instance. Separation experiments such as feature truncation or retraining on a higher-dimensional prior are listed in the outlook (Section~\ref{sec:discussion}). Guardrail per-domain report: the IHDP prior carries a consistent mild cost on sales (worse 5/5) and is directionally neutral-to-favorable on law\_race (better 3/5).

\begin{figure}[t]
\centering
\includegraphics[width=\textwidth]{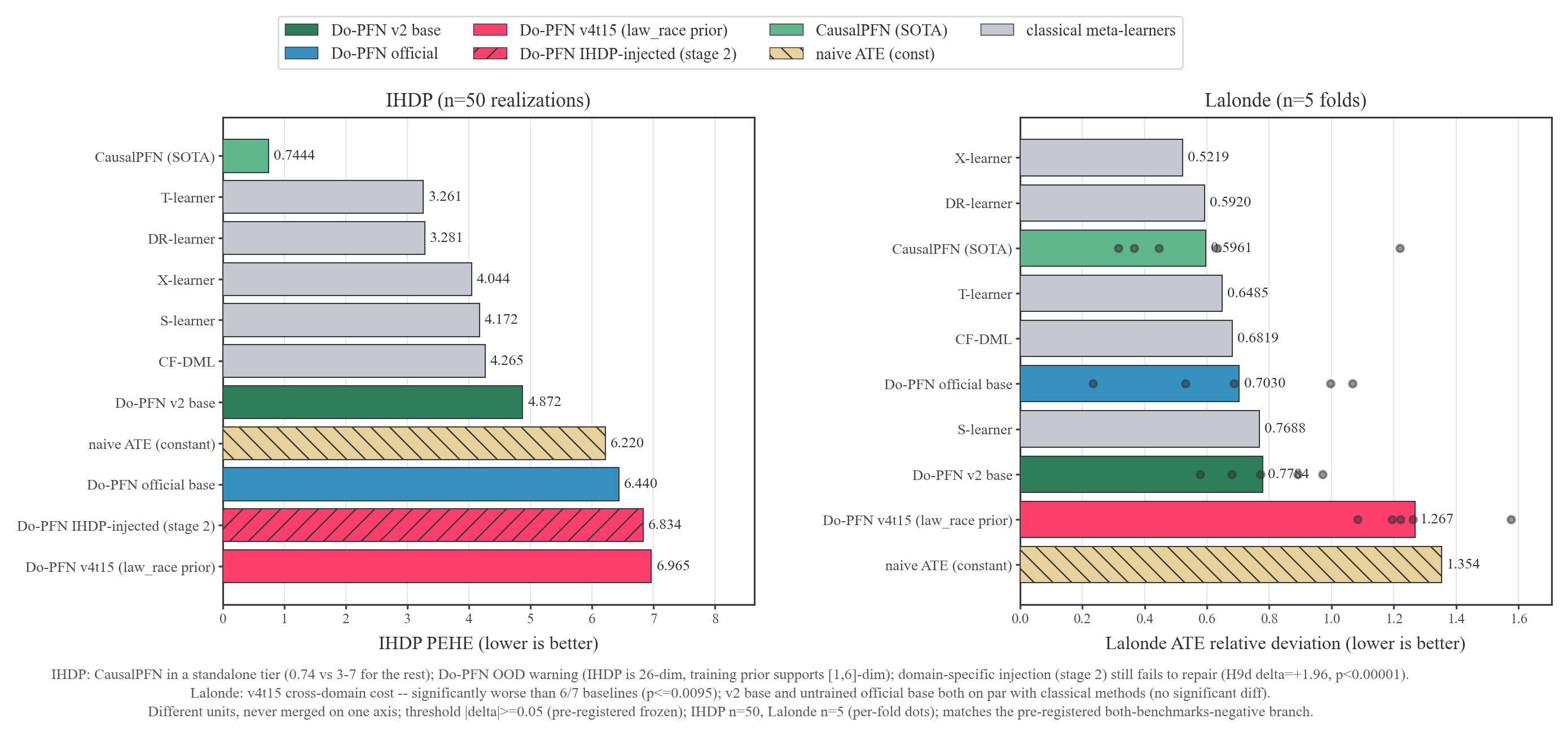}
\caption{M9 standard benchmarks (the two panels have different metric units and are plotted separately): left panel, IHDP PEHE ($n=50$ realization pairing, effect-size threshold $|\Delta|\ge0.05$), CausalPFN in a class of its own ($0.7444$), the three Do-PFN versions in $4.87$--$6.96$, the injected configuration bottom; right panel, Lalonde ATE relative deviation ($n=5$), the two uninjected bases on par with classical methods statistically, the injected configuration significantly worse than 6/7 baselines.}
\label{fig:f11}
\end{figure}

The two honest boundaries point to the same conclusion: \emph{prior-domain--task-domain matching is a necessary but not sufficient condition for injection to help; when the task's dimension exceeds the base's training-prior support, no prior injection can repair it.} Appendix Figure~\ref{fig:f7} gives a conceptual schematic of the four diagnosis--generalization divergence instances.

\begin{figure}[t]
\centering
\includegraphics[width=0.68\textwidth]{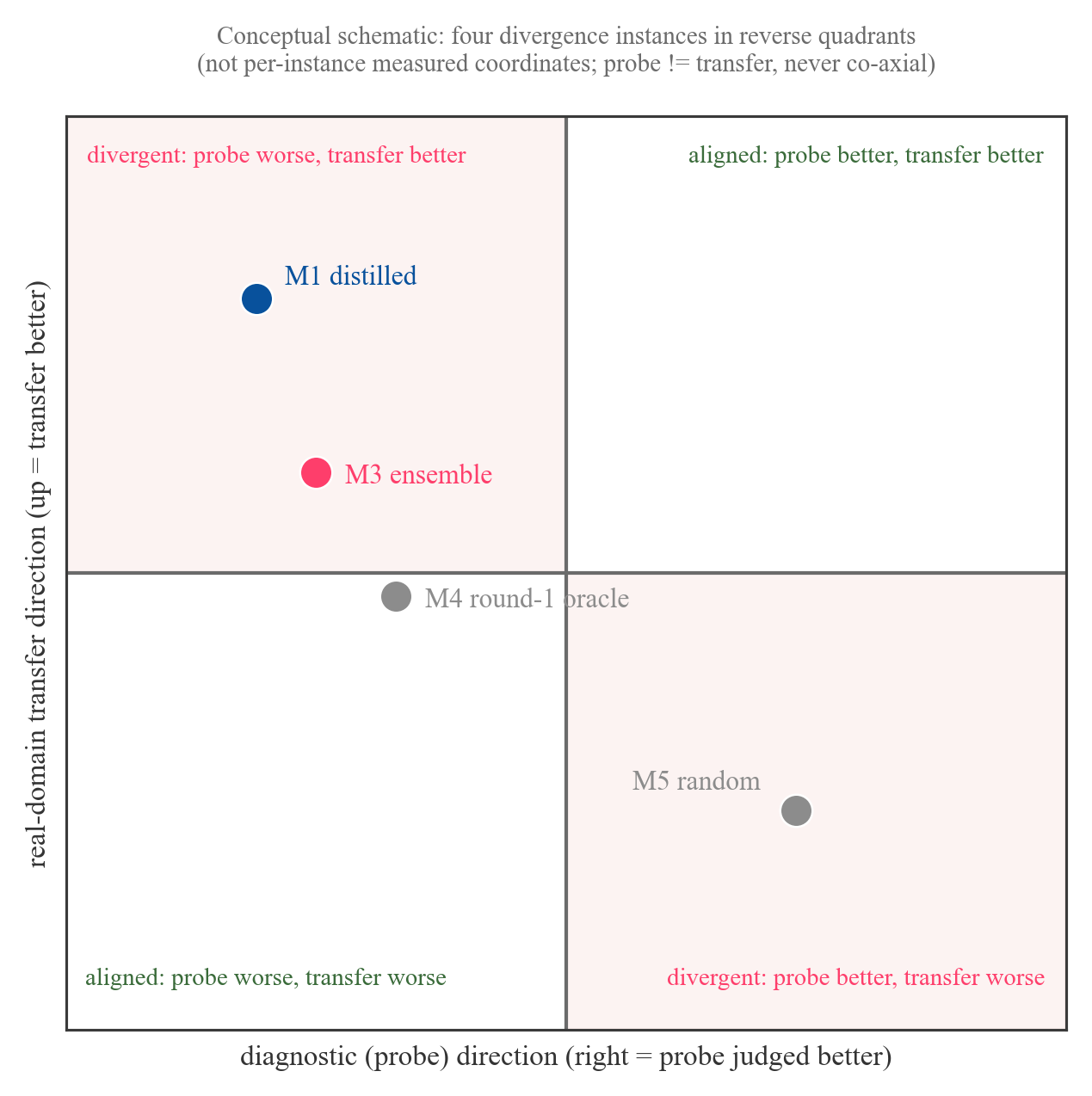}
\caption{Conceptual schematic of the four diagnosis--generalization divergence instances placed in the (probe direction, transfer direction) plane; the reverse quadrants (probe worse but transfer better, and probe better but transfer worse) are shaded. Conceptual schematic, not per-instance measured coordinates.}
\label{fig:f7}
\end{figure}

\FloatBarrier
\section{Discussion}
\label{sec:discussion}

Four threads organize the discussion of the nine modules' evidence: the statement and criteria of the three-condition empirical regularity; the attribution of weakness on standard benchmarks; the general implication of the probe--transfer divergence; and implications for evaluation protocols.

\emph{The three-condition empirical regularity: under what conditions does injection help.} Across the nine modules a pattern repeats, and we induce from it the following empirical regularity (working hypothesis): injection yields significant gains when three conditions hold simultaneously. Condition one: the base is underfit on the task domain. Condition two: the prior domain matches the task domain. Condition three: the task lies within the support of the base's training prior. The evidence supporting the three conditions is tabulated below.

\paragraph{The three-condition empirical regularity: under what conditions does injection help.}
Synthesizing the nine modules, we induce the following empirical regularity (working hypothesis): injection yields significant gains when three conditions hold simultaneously. Condition one, the base is underfit on the task domain; condition two, the prior domain matches the task domain; condition three, the task lies within the support of the base's training prior. The evidence supporting the three conditions is tabulated in Table~\ref{tab:conditions}.

\begin{table}[t]
\centering
\caption{Evidence supporting the three-condition empirical regularity.}
\label{tab:conditions}
\small
\begin{tabular}{@{}p{2.6cm}p{3.8cm}p{3.5cm}p{4.7cm}@{}}
\toprule
Condition & Supporting instance & Evidence & Limitation \\
\midrule
Cond.\ 1 base underfit & v2 base law\_race (M4); v2 base sales (M7) & $2.75\times$ significant (M4, $n=5$); $3.0\times$ directional (M7, $n=3$) & Counterexample is M6 official-base degradation; only two bases, and no protocol re-search on the official base \\
Cond.\ 2 prior-domain match & law\_race prior used on IHDP, Lalonde (M9 stage 1) & Significantly negative contribution on both domains & Only negative instances of domain mismatch; no positive replication of domain match \\
Cond.\ 3 task within support & IHDP 26 dim $\gg$ $[1,6]$ (M9 stage 2) & Domain match still does not repair & Single instance; attribution not separated \\
All three hold & v2 base law\_race (M4) & Significant gain & Only one complete condition-combination cell \\
\bottomrule
\end{tabular}
\end{table}

This induction has important limits. Each condition currently rests on a single instance, and the condition-combination matrix contains almost no test cell in which only one condition is violated; necessity therefore cannot be established. A further complication is that condition two (domain mismatch) and condition three (dimension beyond support) are violated simultaneously in M9: IHDP (26 dimensions) and Lalonde (8 dimensions) both exceed the support, and neither domain matches the law\_race prior, so the two conditions are not separable in the current design. Finally, the sales domain satisfies domain match yet does not reach the official base's native level; satisfying the three conditions therefore does not guarantee that the gain reaches or exceeds a strong base, and the regularity makes no quantitative prediction of gain magnitude. We treat this regularity as an empirical induction and working hypothesis, not an ``if and only if'' theorem. Its operational value is a pre-injection checklist: measure the base's native level on the target domain, check whether the task dimension lies within the training-prior support, and check whether the prior domain matches the task domain.

\paragraph{Why amortized models are weak on standard benchmarks.}
Out-of-support dimensionality is the leading candidate explanation. IHDP's 26-dimensional input far exceeds the $[1,6]$ dimensions supported by Do-PFN's training prior, and that is a boundary of the base architecture, outside the reach of prior injection (Theorem~\ref{thm:support} gives a formal statement of this boundary). The attribution is not separated: no separation experiment has confirmed it yet (see Section~\ref{sec:m9} stage 2), and the fact that the v2 base is significantly stronger than the official base on IHDP suggests that factors beyond dimension also play a role; we have not reconciled this difference. CausalPFN is in a class of its own on IHDP, which suggests better coverage of higher-dimensional input in its training prior or architecture; that gap is for a future direct comparison to confront, and injection by our method will not close it. We state plainly: on standard causal benchmarks, our injection method is not superior to, and is often inferior to, existing methods. What it offers is domain specialization and a complete delineation of the applicability boundary.

\paragraph{General implication of the diagnosis--generalization divergence.}
In four of five experiments, synthetic diagnostic metrics and real-domain generalization give opposite testimony. Their common structure is that diagnostic tasks measure ``fit to the synthetic task distribution, while the real domain measures generalization to the real problem'', and the two align only when the prior and the real distribution are aligned. Theorem~\ref{thm:reversal} proves the existence of this misalignment can exist, and the four divergence instances instantiate it. For the broader synthetic-data community the lesson transfers: capability evidence from a synthetic panel cannot be exchanged automatically for transfer evidence, and any evaluation proxy should pass at least one empirical calibration before it takes on a decision-making role.

\paragraph{Implications for evaluation protocols.}
On several domains the naive ATE constant baseline ranks near the top, which exposes limited individual-effect heterogeneity in semi-synthetic benchmarks. Error-type metrics, to some extent, measure heterogeneity recovery rather than estimation of the average effect. Evaluation-protocol design should keep this in mind.

\paragraph{Limitations.}
The main result rests on a limited set of evaluation domains, bases, and candidate pool; $n$ is between 3 and 5 (IHDP reaches $n=50$); ``content beats diversity'' and the three-condition regularity are evidence chains supported by controlled comparisons, not complete proofs. Injection never surpasses a natively strong base. Results on standard causal benchmarks are negative. Lalonde did not undergo stage-2 domain-specific distillation.

\paragraph{Outlook.}
Three direct extensions follow from this work. Retraining the base on a higher-dimensional prior would remove the leading out-of-support factor. A prior--base matching checker could judge automatically whether the three conditions hold before any injection. Underfit bases are candidates for targeted-repair applications.

\section{Conclusion}

Injecting LLM-distilled causal priors into amortized causal-inference models has so far relied on manual trial and error. This paper makes the decision ``which prior to inject'' as a budget-constrained optimization over a candidate prior pool (Equation~\eqref{eq:selection}) and provides a two-stage closed-loop selection procedure. The framework's winner attains a formally significant $2.75\times$ gain on the primary evaluation domain ($n=5$, $p=0.0086$), and its error level falls below that of the uninjected official base. The latter is a descriptive cross-lineage comparison: the winner is built on the locally retrained base and the official base is not part of the paired test (see Sections~\ref{sec:m4}, \ref{sec:m6}). Adjacent-domain generalization significantly improves at the same time ($p=0.0109$), and no capability degradation is observed (descriptive). Mechanism experiments show the gain comes from the semantic content of the distilled graph; structural diversity alone does not account for it (directional evidence). Nine controlled experiments delineate the applicability boundary of injection and induce a three-condition empirical regularity: when base underfit, prior-domain match, and task within training-prior support all hold simultaneously, injection yields significant gains. The limits of this induction (each condition supported by a single instance; conditions two and three not separable) are made explicit (Section~\ref{sec:discussion}). The framework and this regularity together convert prior injection from manual trial and error into an empirically verifiable selection problem, offering future work a reusable procedure and an operational criterion. The scope is bounded by the current evidence: one base family, two in-house domains, and two standard benchmarks.

\section*{Reproducibility Statement}

This statement addresses the reproducibility gaps raised in review (R2-M9) and records what is documented and what remains to be released.

\paragraph{Base and post-training protocol.}
The base is Do-PFN \cite{robertson2025dopfns} (7.34M parameters). The v2 base is a local retrain that is weaker but plastic; all post-training initializes from v2 and continues for 20k steps at learning rate $5\times10^{-5}$, with these hyperparameters frozen during the platform-diagnosis stage (M0); cheap pre-screening uses 5k steps. The full retraining recipe for v2 (initialization, data source, and hyperparameter table) is documented in the platform-diagnosis report and will be released with the code.

\paragraph{Distillation sources.}
The candidate pool draws on three distillation sources: the primary source V4 (DeepSeek-V4-Flash distillation), a 9B-parameter LLM distillation, and an alternative distillation configuration AESD; decoding temperatures $0.3 / 0.7 / 1.5$. The full model identifiers of the 9B source and the complete definition of AESD will be stated in the released materials; until then they are referred to by these internal labels.

\paragraph{Evaluation domains.}
The primary domain law\_race and the adjacent domain sales are synthetic/proprietary causal benchmarks with ground-truth graphs; their generation mechanism and release policy will be stated in the released materials. The standard benchmarks are public: IHDP (dragonnet NPCI setting A, 50 realizations, 80/20 fixed split, seed = realization id) \cite{hill2011ihdp} and Lalonde (NSW/DW observational sample of 614 rows, DW treated 185 + CPS3 control 429, 5-fold fixed stratified folds, benchmark ATE \$1794.34, 8 covariates frozen) \cite{lalonde1986evaluating,dehejia1999causal}; all arms share the same input (z-score by training-set statistics).

\paragraph{Statistical protocol.}
Significance requires $p<0.05$ plus an effect-size threshold of $|\Delta|\ge0.02$ for M1--M8 and $|\Delta|\ge0.05$ for M9; paired tests pair strictly by seed (or split); $n\le3$ comparisons are reported as directional; standard deviations use ddof=1. All protocols were pre-registered.

\paragraph{Code and data availability.}
The code, distilled-graph specs, candidate-pool definitions, and per-seed result JSONs will be released at submission. CausalPFN's published Table 8 numbers (100 realizations) are used only as an order-of-magnitude reference and are not entered into our tables.

\bibliographystyle{plain}
\bibliography{refs}

\end{document}